\documentclass[10pt,twocolumn,letterpaper]{article}
\usepackage{cvpr}              

\usepackage{graphicx}
\usepackage{amsmath}
\usepackage{amssymb}
\usepackage{amsthm}
\usepackage{booktabs}
\usepackage{appendix}
\usepackage{multirow}
\usepackage{graphicx}
\usepackage{makecell}
\usepackage{xcolor,colortbl}
\usepackage[accsupp]{axessibility}  
\usepackage[linesnumbered,ruled]{algorithm2e}

\usepackage[pagebackref,breaklinks,colorlinks]{hyperref}

\usepackage[capitalize]{cleveref}
\crefname{section}{Sec.}{Secs.}
\Crefname{section}{Section}{Sections}
\Crefname{table}{Table}{Tables}
\crefname{table}{Tab.}{Tabs.}

\newtheorem{theorem}{Theorem}

\newtheorem{lemma}{Lemma}

\theoremstyle{definition}
\newtheorem{definition}{Definition}
\newtheorem{assumption}{Assumption}
\newtheorem{remark}{Remark}

\DeclareMathOperator{\topk}{top} 
\DeclareMathOperator{\blurs}{blurs} 
\DeclareMathOperator{\blur}{blur} 
 
\DeclareMathOperator{\randk}{rand}

\def\cvprPaperID{4552} 
\def\confName{CVPR}
\def\confYear{2023}

\begin{document}

\title{Make Landscape Flatter in Differentially Private Federated Learning}

\author{
Yifan Shi\textsuperscript{\rm 1} 
\quad
Yingqi Liu\textsuperscript{\rm 2}
\quad
Kang Wei\textsuperscript{\rm 2}
\quad
Li Shen\textsuperscript{\rm 3,}\thanks{Corresponding authors: Li Shen and Xueqian Wang}
\quad
Xueqian Wang\textsuperscript{\rm 1,*}
\quad
Dacheng Tao\textsuperscript{\rm 3}
\\
\textsuperscript{\rm 1}Tsinghua University, Shenzhen, China; \textsuperscript{\rm 3}JD Explore Academy, Beijing, China\\
\textsuperscript{\rm 2}Nanjing University of Science and Technology, Nanjing, China\\
{\tt\small shiyf21@mails.tsinghua.edu.cn;
lyq@njust.edu.cn;
kang.wei@njust.edu.cn; 
}\\
{\tt\small mathshenli@gmail.com; wang.xq@sz.tsinghua.edu.cn; dacheng.tao@gmail.com
}
}
\maketitle

\begin{abstract}
To defend the inference attacks and mitigate the sensitive information leakages in Federated Learning (FL), client-level Differentially Private FL (DPFL) is the de-facto standard for privacy protection by clipping local updates and adding random noise. However, existing DPFL methods tend to make a sharper loss landscape and have poorer weight perturbation robustness, resulting in severe performance degradation. To alleviate these issues, we propose a novel DPFL algorithm named DP-FedSAM, which leverages gradient perturbation to mitigate the negative impact of DP. Specifically, DP-FedSAM integrates Sharpness Aware Minimization (SAM) optimizer to generate local flatness models with better stability and weight perturbation robustness, which results in the small norm of local updates and robustness to DP noise, thereby improving the performance. From the theoretical perspective, we analyze in detail how DP-FedSAM mitigates the performance degradation induced by DP. Meanwhile, we give rigorous privacy guarantees with Rényi DP and present the sensitivity analysis of local updates. At last, we empirically confirm that our algorithm achieves state-of-the-art (SOTA) performance compared with existing SOTA baselines in DPFL. Code is available at \url{https://github.com/YMJS-Irfan/DP-FedSAM}
\end{abstract}

\section{Introduction}
Federated Learning (FL) \cite{Li2020federated} allows distributed clients to collaboratively train a shared model without sharing data. 
However, FL faces severe dilemma of privacy leakage \cite{Kairouz2021Advances}.
Recent works show that a curious server can also infer clients’ privacy information such as membership and data features, by well-designed generative models and/or shadow models~\cite{Fredrikson2015model,Shokri2017membership,Melis2018inference,Nasr2019comprehensive,zhang2022fine}. To address this issue, differential privacy (DP) \cite{Dwork2014the} has been introduced in FL, which can protect every instance in any client's data (instance-level DP~\cite{Agarwal2018cpsgd,Hu2021federated, Sun2021federated,Sun2021practical}) or the information between clients (client-level DP~\cite{McMahan2018learning,Robin2017differentially,KairouzL2021the,wei2021user, Rui2022Federated,Anda2022differentially}). In general, client-level DP are more suitable to apply in the real-world setting due to a better model performance. For instance, a language prediction model with client-level DP \cite{geyer2017differentially, McMahan2018learning} has been applied on mobile devices by Google.
In general, the Gaussian noise perturbation-based method is commonly adopted for ensuring the strong client-level DP.
However, this method includes two operations in terms of clipping the $l_2$ norm of local updates to a sensitivity threshold $C$ and adding random noise proportional to the model size, whose standard deviation (STD) is also decided by $C$. These steps may cause severe performance degradation dilemma \cite{cheng2022differentially, hu2022federated}, especially on large-scale complex model \cite{simonyan2014very}, such as ResNet-18 \cite{he2016deep}, or with heterogeneous data.  

\begin{figure*}
\centering
\begin{minipage}[b]{0.55 \linewidth}
    \centering
    \includegraphics[width=1.0\linewidth]{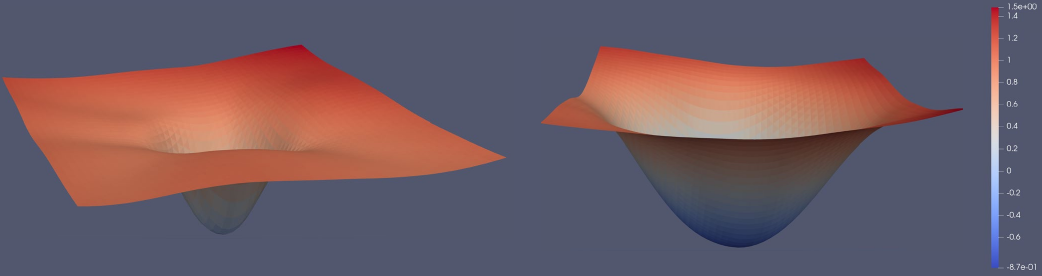}
    \centerline{(a) Loss landscapes.}\medskip
\end{minipage}
\hfill
\begin{minipage}[b]{0.42\linewidth}
    \centering
    \includegraphics[width=1.0\linewidth]{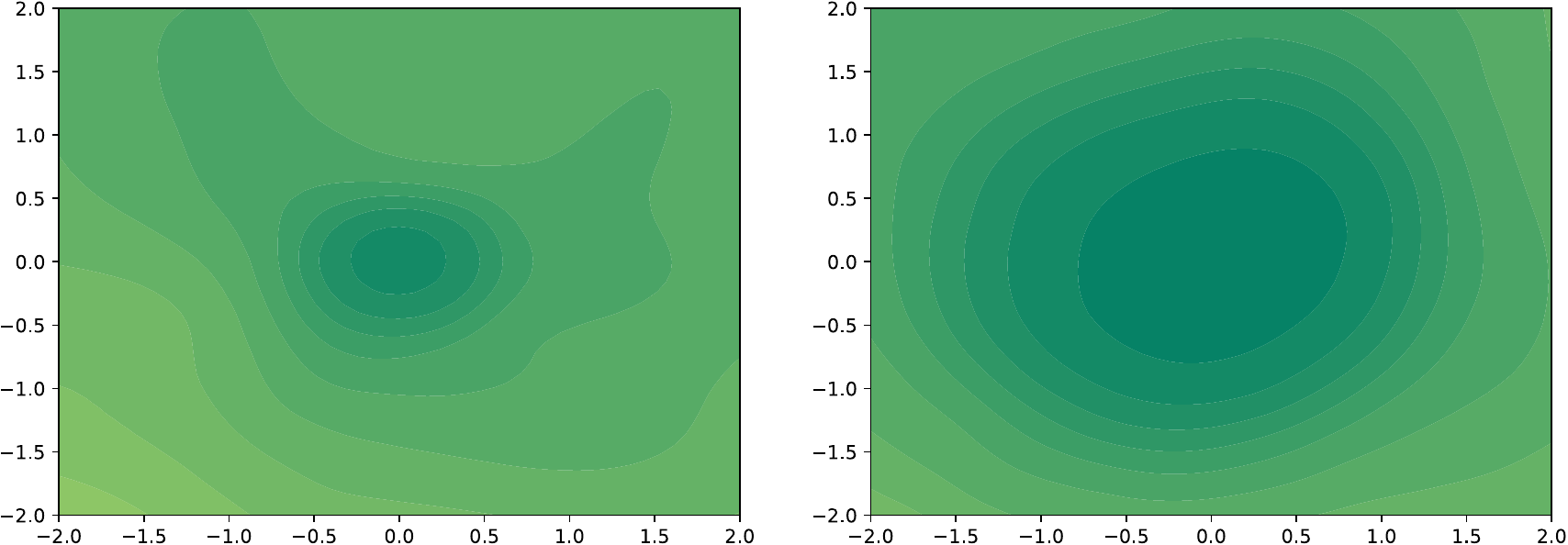}
    \centerline{(b) Loss surface contours.}\medskip
    \label{contour_fedavg_dpfedavg}
\end{minipage}
\vspace{-0.35cm}
\caption{\small Loss landscapes and surface contours comparison between DP-FedAvg (left) and FedAvg (right), respectively.
}
\vspace{-0.4cm}
\label{landscape_fedavg_dpfedavg}
\end{figure*}

The reasons behind this issue may be two-fold: (i) The useful information is dropped due to the clipping operation especially on small $C$, which is contained in the local updates; (ii) The model inconsistency among local models is exacerbated as the addition of random noise severely damages local updates and leads to large variances between local models especially on large $C$ \cite{cheng2022differentially}. To overcome these issues, existing works improve the performance via restricting the norm of local update \cite{cheng2022differentially} and leveraging local update sparsification technique \cite{cheng2022differentially, hu2022federated} to reduce the adverse impacts of clipping and adding amount of random noise. However, the model performance degradation remains significantly severe compared with FL methods without considering the privacy, e.g., FedAvg \cite{mcmahan2017communication}. 

\noindent
\textbf{Motivation.} To further explore the reasons behind this phenomenon, we compare the structure of loss landscapes and surface contours \cite{li2018visualizing} for FedAvg \cite{mcmahan2017communication} and DP-FedAvg \cite{McMahan2018learning,Robin2017differentially} on partitioned CIFAR-10 dataset \cite{krizhevsky2009learning} with Dirichlet distribution ($\alpha =0.6$) and ResNet-18 backbone \cite{he2016deep} in Figure \ref{landscape_fedavg_dpfedavg} (a) and (b), respectively. Note that the convergence of DP-FedAvg is worse than FedAvg as its loss value is higher after a long communication round. Furthermore, FedAvg has a flatter landscape, whereas DP-FedAvg has a sharper one, resulting in poorer generalization ability (sharper minima, see Figure \ref{landscape_fedavg_dpfedavg} (a)) and weight perturbation robustness (see Figure \ref{landscape_fedavg_dpfedavg} (b)), which is caused by the clipped local update information and exacerbated model inconsistency, respectively. Based on these observations, an interesting research question is: \emph{can we further overcome the severe performance degradation via making landscape flatter? }

To answer this question, we propose DP-FedSAM via gradient perturbation to improve model performance. Specifically, a local flat model is generated by using SAM optimizer \cite{foret2021sharpnessaware} in each client, which leads to better stability. After that, a potentially global flat model can be generated by aggregating several flat local models, which results in higher generalization ability and better robustness to DP noise, thereby significantly improving the performance and achieving a better trade-off between performance and privacy. 
Theoretically, we present a tighter bound  
$\small \mathcal{O}(\frac{1}{\sqrt{KT}}+\frac{\sum_{t=1}^T(\overline{\alpha}^t\sigma_g^2 + \tilde{\alpha}^t  L^2)}{T^2} + \frac{L^2 \sqrt{T}\sigma^2C^2d}{m^2\sqrt{K}} )$ for DP-FedSAM in the stochastic non-convex setting, where both $\frac{1}{T}\sum_{t=1}^T\overline{\alpha}^t$ and $\frac{1}{T}\sum_{t=1}^T\tilde{\alpha}^t$ are bounded constant, $K$ and $T$ is local iteration steps and communication rounds, respectively. Meanwhile, we present how SAM mitigates the impact of DP. For the clipping operation, DP-FedSAM reduces the $l_2$ norm and the negative impact of the inconsistency among local updates on convergence. For adding noise operation, we obtain better weight perturbation robustness for reducing the performance damage caused by random noise, thereby having better robustness to DP noise. 
Empirically, we conduct extensive experiments on EMNIST, CIFAR-10,
and CIFAR-100 datasets in both the independently and identically distributed (IID) and Non-IID settings. Furthermore, we observe the loss landscape, surface contour, and the norm distribution of local updates for exploring the intrinsic effects of SAM with DP, which together with the theoretical analysis confirms the effect of DP-FedSAM.

\noindent
\textbf{Contribution.} The main contributions of our work are summarized as four-fold: \textbf{(i)} We propose a novel scheme DP-FedSAM in DPFL to alleviate performance degradation issue from the optimizer perspective.
\textbf{(ii)} We establish an improved convergence rate, which is tighter than the conventional bounds \cite{cheng2022differentially,hu2022federated} in the stochastic non-convex setting. Moreover, we provide rigorous privacy guarantees and sensitivity analysis.
\textbf{(iii)} We are the first to in-depth analyze the roles of the on-average norm of local updates $\overline{\alpha}^{t}$ and local update consistency among clients $\tilde{\alpha}^t$ on convergence.
\textbf{(iv)}
We conduct extensive experiments to verify the effect of DP-FedSAM, which could achieve state-of-the-art (SOTA) performance compared with several strong DPFL baselines.


\vspace{-0.15cm}

\section{Related work}
\noindent
\textbf{Client-level DPFL.}
Client-level DPFL is the de-facto approach for protecting any client's data. DP-FedAvg \cite{dpfedavg} is the first to study in this setting, which trains a language prediction model in a mobile keyboard and ensures client-level DP guarantee by employing the Gaussian mechanism and composing privacy guarantees.
After that, the work in \cite{Kairouz2021TheDD,Thakkar2019adaclip} presents a comprehensive end-to-end system, which appropriately discretizes the data and adds discrete Gaussian noise before performing secure aggregation. Meanwhile, AE-DPFL \cite{Zhu2020VotingbasedAF} leverages the voting-based mechanism among the data labels instead of averaging the gradients to avoid dimension dependence
and significantly reduce the communication cost. Fed-SMP \cite{hu2022federated} uses Sparsified Model Perturbation (SMP) to mitigate the impact of privacy protection on model accuracy. Different from the aforementioned methods, a recent study \cite{cheng2022differentially} revisits this issue and leverages Bounded Local Update Regularization (BLUR) and Local Update Sparsification (LUS) to restrict the norm of local updates and reduce the noise size before executing operations that guarantee DP, respectively. Nevertheless, severe performance degradation remains. 
%



\noindent
\textbf{Sharpness Aware Minimization (SAM).} 
SAM \cite{foret2021sharpnessaware} is an effective optimizer for training deep learning (DL) models, which leverages the flatness geometry of the loss landscape to improve model generalization ability. Recently, \cite{Andriushchenko2022Towards} study the properties of SAM and provide convergence results of SAM for non-convex objectives. As a powerful optimizer, SAM and its variants have been applied to various machine learning (ML) tasks \cite{Zhao2022Penalizing,kwon2021asam,du2021efficient,liu2022towards,Abbas2022Sharp-MAML,mi2022make,zhong2022improving, huangrobust,sunfedspeed,sun2023adasam,shi2023improving}. Specifically, \cite{Qu2022Generalized}, \cite{sunfedspeed} and \cite{Caldarola2022Improving} integrate SAM to improve the generalization, and thus mitigate the distribution shift problem and achieve a new SOTA performance for FL. However, to the best of our knowledge, no efforts have been devoted to the empirical performance and theoretical analysis of SAM in DPFL.

The most related works to this paper are DP-FedAvg in \cite{McMahan2018learning}, Fed-SMP in \cite{hu2022federated}, and DP-FedAvg with LUS and BLUR \cite{cheng2022differentially}. However, these works still suffer from inferior performance due to the exacerbated model inconsistency issue among the clients caused by random noise. Therefore, different from existing works, we try to alleviate this issue by making the landscape flatter and weight perturbation ability more robust. Furthermore, another related work is FedSAM \cite{Qu2022Generalized}, which integrates SAM optimizer to enhance the flatness of the local model and achieves new SOTA performance for FL. On top of the aforementioned studies, we are the first to extend the SAM optimizer into the DPFL setting to effectively alleviate the performance degradation problem. Meanwhile, we provide the theoretical analysis for sensitivity, privacy, and convergence in the non-convex setting. Finally, we empirically confirm our theoretical results and the superiority of performance compared with existing SOTA methods in DPFL.

\section{Preliminary}

In this section, we first give the problem setup of FL, and then introduce the several terminologies in DP. 

\subsection{Federated Learning}
Consider a general FL system consisting of $M$ clients, in which each client owns its local dataset.
Let $\mathcal D^{\text{train}}_i$, $\mathcal D^{\text{val}}_i$ and $\mathcal D^{\text{test}}_i$ denote the training dataset, validation dataset and testing dataset, held by client $i$, respectively, where $i\in \mathcal{U} = \{1, 2,\ldots, M\}$.
Formally, the FL task is expressed as:
\begin{equation}
\small
\mathbf{w}^{\star} = \mathop{\arg\min}_{\mathbf{w}}\sum_{i\in \mathcal{U}}p_{i}f_i(\mathbf{w}, \mathcal D^{\text{train}}_{i}),
\end{equation}
where $p_{i} = \vert \mathcal D^{\text{train}}_i\vert/\vert \mathcal D^{\text{train}}\vert\geq 0$ with $\sum_{i\in \mathcal{U}}{p_{i}}=1$, and $f_i(\cdot)$ is the local loss function with $f_i(\mathbf{w}) = F_i(\mathbf{w}; \xi_i)$, $\xi_i$ is a batch sample data in client $i$. $\vert \mathcal D^{\text{train}}_{i}\vert $ is the size of training dataset $\mathcal D^{\text{train}}_i$ and $\vert \mathcal D^{\text{train}}\vert = \sum_{i\in \mathcal{U}}{\vert \mathcal D_{i}^{\text{train}}\vert}$ is the total size of training datasets, respectively.
For the $i$-th client, a local model is learned on its private training data $\mathcal D^{\text{train}}$ via:
\begin{equation}
\small
\mathbf{w}_{i}=\mathbf{w}_{i}^{t}-\eta \nabla f_i(\mathbf{w}_{i}, \mathcal D^{\text{train}}_{i}).
\end{equation}
Generally, the local loss function $f_i(\cdot)$ has the same expression across each client.
Then, the $M$ associated clients collaboratively learn a global model $\mathbf{w}$ over the heterogeneous training data $\mathcal D^{\text{train}}_{i}$, $\forall i \in \mathcal{U}$.

\subsection{Differential Privacy}
Differential Privacy (DP) \cite{Dwork2014the} is a rigorous privacy notion for measuring privacy risk. In this paper, we consider a relaxed version: Rényi DP (RDP)~\cite{mironov2017renyi}.
\begin{definition}
(Rényi DP, \cite{mironov2017renyi}). Given a real number $\alpha \in (1, \infty )$ and privacy parameter $\rho \ge 0$, a randomized mechanism $\mathcal{M}$ satisfies $(\alpha, \rho)$-RDP if for any two neighboring datasets $\mathcal D$, $\mathcal D'$ that differ in a single record, the Rényi $\alpha$-divergence between $\mathcal{M}(\mathcal D)$ and $\mathcal{M}(\mathcal D^{\prime})$ satisfies:
\begin{equation}
\small
\!\!\!D_{\alpha}\left[\mathcal{M}(\mathcal D) \| \mathcal{M}\left(\mathcal D^{\prime}\right)\right]\!:=\!\frac{1}{\alpha\!-\!1} \log \mathbb{E}\left[\left(\frac{\mathcal{M}(\mathcal D)}{\mathcal{M}\left(\mathcal D^{\prime}\right)}\right)^{\alpha}\!\right] \!\leq \!\rho,
\end{equation}
where the expectation is taken over the output of $\mathcal{M}(\mathcal D^{\prime})$.
\end{definition}
Rényi DP is a useful analytical tool to measure the privacy and accurately represent guarantees on the tails of the privacy loss, which is strictly stronger than $(\epsilon, \delta)$-DP for $\delta > 0 $. Thus, we provide the privacy analysis based on this tool for each user's privacy loss.

\begin{definition}
\text{\rm ($l_2$ Sensitivity).}\label{sensitivity}
\cite[Definition 2]{cheng2022differentially}
Let $\mathcal{F}$ be a function, the $L_{2}$-sensitivity of $\mathcal{F}$ is defined as $\mathcal{S}=\max _{D \simeq D^{\prime} } \| \mathcal{F}\left(D\right)-$ $\mathcal{F}\left(D^{\prime}\right) \|_{2}$, where the maximization is taken over all pairs of adjacent datasets.
\end{definition}

The sensitivity of a function $\mathcal{F}$ captures the magnitude which a single individual’s data can change the function $\mathcal{F}$ in the worst case.
Therefore, it plays a crucial role in determining the magnitude of noise required to ensure DP.

\begin{definition}(Client-level DP) \cite[Definition 1]{mcmahan2017learning}. 
 A randomized algorithm $\mathcal{M}$ is $(\epsilon, \delta)$-DP if for any two adjacent datasets $U$, $U^{\prime}$ constructed by adding or removing all records of any client, and every possible subset of outputs $O$ satisfy the following inequality:
\begin{equation}
\operatorname{Pr}[\mathcal{M}(U) \in O] \leq e^{\epsilon} \operatorname{Pr}\left[\mathcal{M}\left(U^{\prime}\right) \in O\right]+\delta.
\end{equation}
\end{definition}

In client-level DP, we aim to ensure participation information for any clients. Therefore, we need to make local updates similar whether one client participates or not.


\section{Methodology}

To revisit the severe performance degradation challenge in DPFL, we observe the loss landscapes and surface contours of FedAvg and DP-FedAvg in Figure \ref{landscape_fedavg_dpfedavg} (a) and (b), respectively. And we find that the DPFL method has a sharper landscape with both poorer generalization ability and weight perturbation robustness than the FL method.
It means that the DPFL method may result in poor flatness and make model sensitivity to noise perturbation. However, the study focusing on this issue remains unexplored.
Therefore, we plan to face this challenge from the optimizer perspective by adopting a SAM optimizer in each client, dubbed DP-FedSAM, whose local loss function is defined as:
\begin{equation} \small\label{Eq:sam}
\small
    f_i(\mathbf{w}) = \mathbb{E}_{\xi\sim \mathcal{D}_i}\max_{\|\delta_i\|_2 \leq \rho} F_i(\mathbf{w}^{t,k}(i) +\delta_i; \xi_i), \quad i \in \mathcal{N},
\end{equation}
where $\mathbf{w}^{t,k}(i) +\delta_i$ is viewed as the perturbed model, $\rho$ is a predefined constant controlling the radius of the perturbation and $\|\cdot\|_2$ is a $l_2$-norm.
Instead of searching for a solution via SGD \cite{bottou2010large,bottou2018optimization}, SAM \cite{foret2021sharpnessaware} aims to seek a solution in a flat region by adding a small perturbation, i.e., $w + \delta$ with more robust performance. 
Specifically, for each client $i\in\{1,2,...,M\}$, each local iteration $k \in \{0,1,...,K-1\}$ in each communication round $t \in \{0,1,...,T-1\}$, the $k$-th inner iteration in client $i$ is performed as:
\begin{gather}
\small
        \mathbf{w}^{t,k+1}(i)=\mathbf{w} ^{t,k}(i)-\eta \tilde{\mathbf{g}}^{t,k}(i),  \label{local iteration} \\
       \!\!\!\! \tilde{\mathbf{g}}^{t,k}(i)=\nabla F_i(\mathbf{w}^{t,k} + \delta(\mathbf{w}^{t,k});\xi),  \delta(\mathbf{w}^{t,k})=\frac{\rho \mathbf{g}^{t,k}}{\left \| \mathbf{g}^{t,k} \right \|_2} .\label{g} \!\!\!
\end{gather}
Where $\delta(\mathbf{w}^{t,k})$ is calculated by using first-order Taylor expansion around $\mathbf{w}^{t,k}$ \cite{foret2021sharpnessaware}.
After that,
we adopt a sampling mechanism, gradient clipping, and Gaussian noise adding to ensure client-level DP. Note that this sampling can amplify the privacy guarantee since it decreases the chances of leaking information about a particular individual \cite{Abadi2016Deep}. Specifically,
after sampling $m$ clients with probability $q=m/M$ at each communication round, which is important for measuring privacy loss. 
Then, we clip the local updates first:
\begin{equation} \small
\small
\tilde{\Delta}^t_{i} = \Delta_{i}^{t}\cdot\min\left(1, \frac{C}{\Vert \Delta_{i}^{t}\Vert_2}\right). 
\end{equation}
After clipping, we add Gaussian noise to the local update for ensuring client-level DP as follows:
\begin{equation}\label{dp_c}
\small
\hat{\Delta}^t_{i} = \tilde{\Delta}^t_{i} + \mathcal{N}(0, \sigma^2C^2 \cdot \mathbf{I}_d/m).
\end{equation}
With a given noise variance $\sigma^2$, the accumulative privacy budget $\epsilon$ can be calculated based on the sampled Gaussian mechanism~\cite{Yousefpour2021Opacus}. 
Finally, we summarize the training procedures of DP-FedSAM in Algorithm \ref{DFedAvg_DP}.

\begin{algorithm}[ht]
\small
\caption{DP-FedSAM}
\label{DFedAvg_DP}
\SetKwData{Left}{left}\SetKwData{This}{this}\SetKwData{Up}{up} \SetKwFunction{Union}{Union}\SetKwFunction{FindCompress}{FindCompress}
\SetKwInOut{Input}{Input}\SetKwInOut{Output}{Output}
\Input{Total number of clients $M$, sampling ratio of clients $q$, total number of communication rounds $T$, the clipping threshold $C$, local learning rate $\eta$, and total number of the local iterates are $K$.} 
\Output{Generate global model $\mathbf{w}^{T}$.}
\textbf{Initialization:} Randomly initialize the global model $\mathbf{w}^{0}$.\\
\For{$t=0$ \KwTo $T-1$}{
    Sample a set of $m=qM $ clients at random without replacement, denoted by $\mathcal{W}^t $.
    
    \For{client $i=1$ \KwTo $m$ \emph{\textbf{in parallel}} }{
        \For{$k=0$ \KwTo $K-1$ }{
        Update the global parameter as the local parameter
        $\mathbf{w}^{t}(i) \gets \mathbf{w}^{t}$.
        
        Sample a batch of local data $\xi_i$ and calculate local gradient $\mathbf{g}^{t,k}(i)=\nabla F_i(\mathbf{w}^{t,k}(i);\xi_i)$.
        
        Gradient perturbation by Equation (\ref{g}).
        
        Local iteration update by Equation (\ref{local iteration}).
        }
        $ \Delta ^t_i = \mathbf{w}^{t,K}(i) - \mathbf{w}^{t,0}(i)$.
        
        Clip and add noise for DP by Equation (\ref{dp_c}).
        
        \textbf{Return} $\hat{\Delta} ^t(i)$.
    }
    $\mathbf{w}^{t+1} \gets \mathbf{w}^{t} + \frac{1}{m}\sum_{i\in \mathcal{W}^t} \hat{\Delta} ^t(i) $.
}
\end{algorithm}
Compared with existing DPFL methods \cite{McMahan2018learning,hu2022federated,cheng2022differentially},
the benefits of DP-FedSAM lie in three-fold: (i) We introduce SAM into DPFL to alleviate severe performance degradation via seeking a flatness model in each client, which is caused by the exacerbated inconsistency of local models. Specifically, DP-FedSAM generates both better generalization ability and robustness to DP noise by making the global model flatter. 
(ii) In addition, we analyze in detail how DP-FedSAM mitigates the negative impacts of DP. Specifically, we theoretically analyze the convergence with the on-average norm of local updates $\overline{\alpha}^{t}$ and local update consistency among clients $\tilde{\alpha}^t$, and empirically confirm these results via observing the norm distribution and average norm of local updates (see Section \ref{exper_DP}). (iii) Meanwhile, we present the theories unifying the impacts of gradient perturbation $\rho$ in SAM, the on-average norm of local updates $\overline{\alpha}^{t}$, and local update consistency among clients $\tilde{\alpha}^t$ in clipping operation, and the variance of random noise $\sigma^2C^2/m$ upon the convergence rate in Section \ref{th}.
\section{Theoretical Analysis}\label{th}

In this section, we give a rigorous analysis of DP-FedSAM, including its sensitivity, privacy, and convergence rate. The detailed proof is placed in \textbf{Appendix} \ref{appendix_th}. Below, we first give several necessary assumptions.

\begin{assumption} \label{a1}
(Lipschitz smoothness). The function $F_i$ is differentiable and $\nabla F_i$ is $L$-Lipschitz continuous, $\forall i \in \{1,2,\ldots,M\}$, i.e.,
$\|\nabla F_i({\bf x}) - \nabla F_i({\bf y})\| \leq L \|{\bf x} - {\bf y}\|,$
for all ${\bf x}, {\bf y} \in \mathbb{R}^d$.
\end{assumption}

\begin{assumption} \label{a2}
(Bounded variance). The gradient of the function $f_i$ have $\sigma_l$-bounded variance, i.e.,
$\mathbb{E}_{\xi_i}\left\|\nabla F_i (\mathbf{w}^k(i);\xi_i ) -\nabla F_i (\mathbf{w}(i))\right  \|^2 \leq \sigma_l^2$, $\forall i \in \{1,2,\ldots,M\}, k \in \{1, ..., K-1\},$
the global variance is also bounded, i.e., $\frac{1}{M} \sum_{i=1}^M \|\nabla f_i({\bf w}) - \nabla f({\bf w})\|^2 \leq \sigma_{g}^2$ for all ${\bf w} \in \mathbb{R}^d$. It is not hard to verify that the $\sigma_g$ is smaller than the homogeneity parameter $\beta$, i.e., $\sigma_g^2 \leq \beta^2$.
\end{assumption}

\begin{assumption}\label{a3}
(Bounded gradient). For any $i \!\in\! \{1,2,\ldots,M\}$ and ${\bf w}\!\in\! \mathbb{R}^d$, we have $\|\nabla f_i({\bf w})\|\!\leq\! B $.
\end{assumption}

\begin{assumption}\label{a4}
(Unbiased Gradient Estimator). For any data sample $z$ from $\mathcal{D}_i$ and ${\bf w}\!\in\! \mathbb{R}^d$, the local gradient estimator is unbiased, i.e., $\mathbb{E}[\nabla f_i(\mathbf{w}; z)]= \mathbb{E}[\nabla f_i(\mathbf{w})]$.
\end{assumption}

Note that the above assumptions are mild and commonly used in characterizing the convergence rate of FL \cite{Sun2022Decentralized,shi2023improving,ghadimi2013stochastic,yang2021achieving,bottou2018optimization,reddi2020adaptive, huang2022achieving, Qu2022Generalized, cheng2022differentially,hu2022federated}. 
Furthermore, gradient clipping operation in DL is often used to prevent the gradient explosion phenomenon, thereby the gradient is bounded. 
The technical difficulty for DP-FedSAM lies in: (i) how SAM mitigates the impact of DP; (ii) how to analyze in detail the impacts of the consistency among clients and the on-average norm of local updates caused by clipping operation.  


\subsection{Sensitivity Analysis}
At first, we study the sensitivity of local update  $\Delta^t_i$ from any client $i \in \{1,2,..., M\}$ before clipping at $t$-th communication round in DP-FedSAM. This upper bound of sensitivity can roughly measure the degree of privacy protection.
Under Definition \ref{sensitivity}, the sensitivity can be denoted by $\mathcal{S}_{\Delta_i^t}$ in client $i$ at $t$-th communication round. 
\begin{theorem} (Sensitivity).\label{th:sensitivity}
Denote $\Delta_i^t(\bf x)$ and $\Delta_i^t(\bf y)$ as the local update at $t$-th communication round, the model $\mathbf{x}(i)$ and $\mathbf{y}(i)$ is conducted on two sets which differ at only one sample. Assume the initial model parameter $\mathbf{w}^t(i)=\mathbf{x}^{t, 0}(i) =\mathbf{y}^{t, 0}(i)$. With the above assumptions, the expected squared  
sensitivity $\mathcal{S}^2_{\Delta_i^t}$ of local update is upper bounded,
\begin{align}
\small
    \mathbb{E}\mathcal{S}^2_{\Delta_i^t} \leq
     \frac{6\eta^2\rho^2KL^2(12K^2L^2\eta^2+ 10)}{1-2\eta^2L^2 K}
\end{align}
When the local adaptive learning rate satisfies $\eta=\mathcal{O}({1}/{L\sqrt{KT}})$ and the perturbation amplitude $\rho$
proportional to the learning rate, e.g., $\rho = \mathcal{O}(\frac{1}{\sqrt{T}})$, we have
\begin{align}
\small
    \mathbb{E}\mathcal{S}^2_{\Delta_i^t} \leq
    \mathcal{O}\left(\frac{1}{T^2}\right). 
\end{align}
For comparison, we also present the expected squared sensitivity of local update with SGD in DPFL, that is $ \mathbb{E}\mathcal{S}^2_{\Delta_i^t, SGD} \leq \frac{6\eta^2\sigma_l^2K}{1-3\eta^2KL^2}$.
Thus $\mathbb{E}\mathcal{S}^2_{\Delta_i^t, SGD} \leq \mathcal{O}(\frac{\sigma_l^2}{KL^2T})$ when $\eta=\mathcal{O}({1}/{L\sqrt{KT}})$.
\begin{remark}
It is clearly seen that the upper bound in $  \mathbb{E}\mathcal{S}^2_{\Delta_i^t, SAM}$ is tighter than that in $\mathbb{E}\mathcal{S}^2_{\Delta_i^t, SGD}$. From the perspective of privacy protection, it means DP-FedSAM has a better privacy guarantee than DP-FedAvg.
Another perspective from local iteration, which means both better model consistency among clients and training stability.
\end{remark}
\end{theorem}


\subsection{Privacy Analysis}
To achieve client-level privacy protection, we derive the sensitivity of the aggregation process at first after clipping the local updates.
\begin{lemma}
The sensitivity of client-level DP in DP-FedSAM can be expressed as $C/m$.
\end{lemma}
\begin{proof}
Given two adjacent batches $\mathcal{W}^{t}$ and $\mathcal{W}^{t,\rm{adj}}$, that $\mathcal{W}^{t,\rm{adj}}$ has one more or less client, we have
\begin{equation} \small
\small
\left\Vert \frac{1}{m}\sum_{i\in \mathcal{W}^{t}}\Delta^{t}_{i}-\frac{1}{m}\sum_{j\in \mathcal{W}^{t,\rm{adj}}} \Delta^{t}_{j} \right\Vert_{2} = \frac{1}{m}\left\Vert \Delta^{t}_{j'} \right\Vert_{2}\leq \frac{C}{m},
\end{equation}
where $\Delta^{t}_{j'}$ is the local update of one more or less client.
\end{proof}
\begin{remark}
The value of this sensitivity can determine the amount of variance for adding random noise.
\end{remark}
After adding Gaussian noise, we calculate the accumulative privacy budget~\cite{Yousefpour2021Opacus} along with training as follows. 
\begin{theorem} \label{th:privacy}
After $T$ communication rounds, the accumulative privacy budget is calculated by:
\begin{equation} \small\label{eq:accumulative_eps}
\begin{aligned}
\epsilon = \overline{\epsilon} + \frac{(\alpha-1)\log(1-\frac{1}{\alpha})-\log(\alpha)-\log(\delta)}{\alpha-1},
\end{aligned}
\end{equation}
where
\begin{equation} \small
\begin{aligned}
\overline{\epsilon} &= \frac{T}{\alpha-1}\ln {\mathbb{E}_{z\sim \mu_{0}(z)}\left[\left(1-q+\frac{q \mu_{1}(z)}{\mu_{0}(z)}\right)^{\alpha}\right]},
\end{aligned}
\end{equation}
and $q$ is sample rate for client selection, $\mu_{0}(z)$ and $\mu_{1}(z)$ denote the Gaussian probability density function (PDF) of $\mathcal{N}(0,\sigma)$ and the mixture of two Gaussian distributions $q\mathcal{N}(1,\sigma)+(1-q)\mathcal{N}(0,\sigma)$, respectively, $\sigma$ is the noise STD, $\alpha$ is a selectable variable.
\end{theorem}
\begin{remark}
It can be seen that a small sampling rate $q$ can enhance the privacy guarantee by decreasing the privacy budget, but it may also degrade the training performance due to the number of participating clients being reduced in each communication round. So a better trade-off is needed.
\end{remark}

\subsection{Convergence Analysis}
Below, we give a convergence analysis of how DP-FedSAM mitigates the negative impacts of DP. The technical contribution also is combining the impacts of the on-average norm of local updates $\overline{\alpha}^{t}$ and local update consistency among clients $\tilde{\alpha}^t$ on the rate. Moreover, we also empirically confirm these results in Section \ref{exper_DP}.

\begin{theorem}\label{th:conver}
Under assumptions 1-4, local learning rate satisfies $\eta=\mathcal{O}({1}/{L\sqrt{KT}})$ and $f^{*}$ is denoted as the minimal value of $f$, i.e., $f(x)\ge f(x^*)=f^*$ for all $x\in \mathbb{R}^{d}$. 
When the perturbation amplitude $\rho$ is
proportional to the learning rate, e.g., $\rho = \mathcal{O}(1/\sqrt{T})$,
the sequence of outputs $\{\mathbf{w}^t\}$ generated by Alg. \ref{DFedAvg_DP}, we have:
\begin{equation*}
\small
\begin{split}
\small
    & \frac{1}{T} \sum_{t=1}^T
    \mathbb{E}\left[\overline{\alpha}^{t}\left\|\nabla f\left(\mathbf{w}^{t}\right)\right\|^{2}\right]   \leq  
    \underbrace{\mathcal{O}\left(\frac{2L(f({\bf w}^{1})-f^{*})}{\sqrt{KT}} + \frac{ L^2\sigma_{l}^2}{KT^2}\right)}_{\text{From FedSAM}} \\ 
    &\qquad\quad +
    \underbrace{
     \underbrace{\mathcal{O}\left( \frac{\sum_{t=1}^T(\overline{\alpha}^t  \sigma_{g}^2 + \tilde{\alpha}^t  L^2  )}{T^2}  \right)}_{\text{Clipping}}
    + \underbrace{ \mathcal{O}\left(\frac{L^2 \sqrt{T}\sigma^2C^2d}{m^2\sqrt{K}} \right)}_{\text{Adding noise}}
    }_{\text{From operations for DP}} 
    \end{split}
\end{equation*}
where
\begin{equation} \small
\begin{split}
    \overline{\alpha}^{t} :=\frac{1}{M} \sum_{i=1}^{M} \alpha^t_i~~~ \text{and}  ~~~ \tilde{\alpha}^t :=\frac{1}{M}\sum_{i=1}^{M} |\alpha_i^t - \overline{\alpha_i^t}|, 
\end{split}
\end{equation}
with $\alpha^t_i = \min (1, \frac{C}{ \eta  \|  \sum_{k=0}^{K-1}  \tilde{\mathbf{g}}^{t,k}(i) \|} ) $, respectively. Note that $\overline{\alpha}^{t}$ and $\tilde{\alpha}^t$ measure the on-average norm of local updates and local update consistency among clients before clipping and adding noise operations in DP-FedSAM, respectively.
\end{theorem}
\begin{table*}
\caption{Averaged training accuracy (\%) and testing accuracy (\%) on two data in both IID and Non-IID settings for all compared methods.}
\vspace{-0.2cm}
\centering
\scriptsize
\renewcommand{\arraystretch}{0.5}
\label{table:all_baselines}
\resizebox{0.9\linewidth}{!}{
\begin{tabular}{cccccccc} 
\toprule
\multirow{2}{*}{\textbf{Task}} & \multirow{2}{*}{\textbf{Algorithm}} & \multicolumn{2}{c}{Dirichlet~0.3} & \multicolumn{2}{c}{Dirichlet~0.6} & \multicolumn{2}{c}{IID}  \\ 
\cmidrule{3-8}
                      &                            & Train         & Validation          & Train         & Validation          & Train      & Validation           \\ 
\midrule 
       & DP-FedAvg       & 99.28$\pm$0.02 & 73.10$\pm$0.16          & 99.55$\pm$0.02 & 82.20$\pm$0.35          & 99.66$\pm$0.40 & 81.90$\pm$0.86           \\
       & Fed-SMP-$\randk_k$ & 99.24$\pm$0.02 & 73.72$\pm$0.53          & 99.71$\pm$0.01 & 82.18$\pm$0.73          & 99.71$\pm$0.61 & 84.16$\pm$0.83           \\
       & Fed-SMP-$\topk_k $    & 99.31$\pm$0.04 & 75.75$\pm$0.35          & 99.72$\pm$0.02 & 83.41$\pm$0.91          & 99.73$\pm$0.40 & 83.32$\pm$0.52           \\
EMNIST & DP-FedAvg-$\blur$      & 99.12$\pm$0.02 & 73.71$\pm$0.02          & 99.66$\pm$0.00 & 83.20$\pm$0.01          & 99.67$\pm$0.03 & 82.92$\pm$0.49           \\
       & DP-FedAvg-$\blurs$     & 99.63$\pm$0.08 & 76.25$\pm$0.35          & 99.72$\pm$0.02 & 83.41$\pm$0.91          & 99.74$\pm$0.45 & 82.92$\pm$0.49           \\
       & DP-FedSAM       & 96.28$\pm$0.64 & 76.81$\pm$0.81          & 95.07$\pm$0.45 & 84.32$\pm$0.19 & 95.61$\pm$0.94 & 85.90$\pm$0.72           \\
       & DP-FedSAM-$\topk_k $  & 94.77$\pm$0.11 & \textbf{77.27$\pm$0.67} & 95.87$\pm$1.52 & \textbf{84.80$\pm$0.60 }         & 96.12$\pm$0.85 & \textbf{87.70$\pm$0.83}  \\
\midrule
                      & DP-FedAvg                  & 93.65$\pm$0.47    & 47.98$\pm$0.24          & 93.65$\pm$0.42    & 50.05$\pm$0.47          & 93.65$\pm$0.15 & 50.90$\pm$0.86           \\
                      & Fed-SMP-$\randk_k$              & 95.46$\pm$0.43    & 48.14$\pm$0.12          & 95.36$\pm$0.06    & 51.33$\pm$0.36          & 95.36$\pm$0.06 & 50.61$\pm$0.20           \\
                      & Fed-SMP-$\topk_k $           & 95.49$\pm$0.14    & 49.93$\pm$2.29          & 95.49$\pm$0.09    & 54.11$\pm$0.83          & 95.49$\pm$0.10 & 53.30$\pm$0.45           \\
CIFAR-10              & DP-FedAvg-$\blur$              & 95.47$\pm$0.12    & 47.66$\pm$0.01          & 99.66$\pm$0.42    & 51.05$\pm$0.01          & 94.50$\pm$0.05 & 52.56$\pm$0.47           \\
                      & DP-FedAvg-$\blurs$          & 96.79$\pm$0.51    & 51.23$\pm$0.66          & 99.72$\pm$0.09    & 54.11$\pm$0.83          & 96.45$\pm$0.30 & 53.48$\pm$0.76           \\
                      & DP-FedSAM                  & 90.38$\pm$0.90    & 53.92$\pm$0.55          & 90.83$\pm$0.15    & 54.14$\pm$0.60          & 90.83$\pm$0.16 & 55.58$\pm$0.50           \\
                      & DP-FedSAM-$\topk_k $            & 93.25$\pm$0.60    & \textbf{54.85$\pm$0.86} & 92.60$\pm$0.65    & \textbf{57.00$\pm$0.69} & 91.52$\pm$0.11 & \textbf{58.82$\pm$0.51}  \\
\bottomrule
\end{tabular}}
\vspace{-0.2cm}
\end{table*}

\begin{remark}
The proposed DP-FedSAM can achieve a tighter bound in general non-convex setting compared with previous work \cite{cheng2022differentially, hu2022federated} ($\small \mathcal{O}\left( \frac{1}{\sqrt{KT}} + \frac{6K\sigma_{g}^2 + \sigma_l^2}{T} + \frac{B^2\sum_{t=1}^T(\overline{\alpha}^t + \tilde{\alpha}^t)}{T}+\frac{L^2 \sqrt{T}\sigma^2C^2d}{m^2\sqrt{K}}\right)$ in \cite{cheng2022differentially} and $\small \mathcal{O}\left( \frac{1}{\sqrt{KT}} + \frac{3\sigma_{g}^2 + 2\sigma_l^2}{\sqrt{KT}} +\frac{4L^2 \sqrt{T}\sigma^2C^2d}{m^2\sqrt{K}}\right)$ in \cite{hu2022federated}) and reduce the impacts of the local and global variance $\sigma_l^2$, $\sigma_g^2$. 
Meanwhile, 
we are the first to theoretically analyze these two impacts of both the on-average norm of local updates $\overline{\alpha}^t$ and local update inconsistency among clients $\tilde{\alpha}^t$ on convergence. And the negative impacts of $\overline{\alpha}^t$ and $\tilde{\alpha}^t$ are also significantly mitigated upon convergence compared with previous work \cite{cheng2022differentially} due to local SAM optimizer being adopted.
It means that we effectively alleviate performance degradation caused by clipping operation in DP and achieve better performance under symmetric noise. This theoretical result has also been empirically verified on several real-world data (see Section \ref{eva} and \ref{exper_DP}).

\end{remark}

\section{Experiments}
In this section, we conduct extensive experiments to verify the effectiveness of DP-FedSAM. 

\subsection{Experiment Setup}

\noindent
\textbf{Dataset and Data Partition.}\  
The efficacy of DP-FedSAM is evaluated on three datasets, including \textbf{EMNIST} \cite{cohen2017emnist}, \textbf{CIFAR-10} and \textbf{CIFAR-100} \cite{krizhevsky2009learning}, in both IID and Non-IID settings. Specifically, Dir Partition \cite{hsu2019measuring} is used for simulating Non-IID across federated clients, where the local data of each client is partitioned by splitting the total dataset through sampling the label ratios from the Dirichlet distribution Dir($\alpha$) with parameters $\alpha=0.3$ and $\alpha=0.6$.

\noindent
\textbf{Baselines.} We focus on DPFL methods that ensure client-level DP. Thus we compare DP-FedSAM with existing DPFL baselines: \textbf{DP-FedAvg} \cite{McMahan2018learning} ensures client-level DP guarantee by directly employing Gaussian mechanism to the local updates. \textbf{DP-FedAvg-blur} \cite{cheng2022differentially} adds regularization method (BLUR) based on DP-FedAvg. \textbf{DP-FedAvg-blurs} \cite{cheng2022differentially} uses local update sparsification (LUS) and BLUR for improving the performance of DP-FedAvg.
\textbf{Fed-SMP-$\randk_k$} and \textbf{Fed-SMP-$\topk_k$} \cite{hu2022federated} leverage random sparsification and $\topk_k$ sparsification technique for reducing the impact of DP noise on model accuracy, respectively.

\noindent
\textbf{Configuration.}
For EMNIST, we use a simple CNN model and train $200$ communication round. For CIFAR-10 and CIFAR-100 datasets, we use the ResNet-18 \cite{he2016deep} backbone and train $300$ communication round. 
For all experiments, we set the number of clients $M$ to $500$. The default sample ratio $q$ of the client is $0.1$. The local learning rate $\eta$ is set to 0.1 with a decay rate $0.005$ and momentum $0.5$, and the number of training epochs is $30$. For privacy parameters, noise multiplier $\sigma$ is set to $0.95$ and  the privacy failure probability $\delta=\frac{1}{M}$. The clipping threshold $C$ is selected by grid search from set $\{0.1,0.2,0.4,0.6,0.8\}$, and we find that the gradient explosion phenomenon will occur when $C \ge 0.6$ on EMNIST and $C=0.2$ performs better on three datasets.  The weight perturbation ratio is set to $\rho = 0.5$. We run each experiment $3$ trials and report the best-averaged testing accuracy in each experiment.

The results on EMNIST and CIFAR-10 datasets are placed in the main paper, while the results on  CIFAR-100 are placed in \textbf{Appendix} \ref{imp_detail}.
For a more comprehensive and fair comparison, we integrate DP-FedSAM with $\topk_k$ sparsification technique \cite{hu2022federated, cheng2022differentially}, named DP-FedSAM-$\topk_k$ (More discussion in \textbf{Appendix} \ref{DP-sam-topk}), to compare with baselines that also use local sparsification technique, such as Fed-SMP-$\randk_k$, Fed-SMP-$\topk_k$, and DP-FedAvg-$\blurs$.

\begin{figure*}[t]
\centering
    \begin{subfigure}{1\linewidth}
    \centering
        \includegraphics[width=0.95\textwidth]{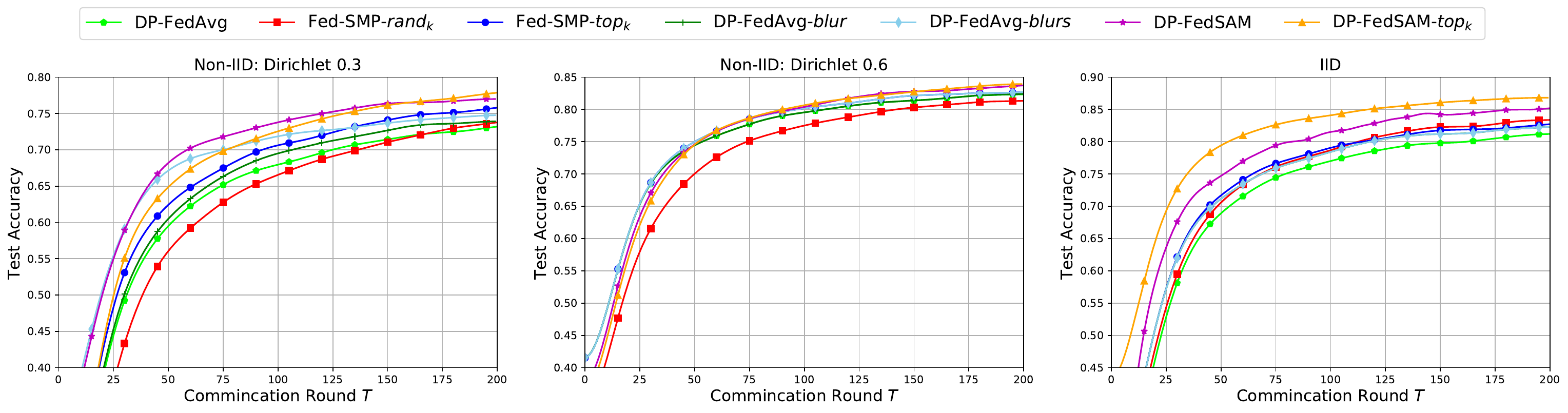}
        \caption{EMNIST}
        \label{fig:emnist}
    \end{subfigure}
    \begin{subfigure}{1\linewidth}
    \centering
        \includegraphics[width=0.95\textwidth]{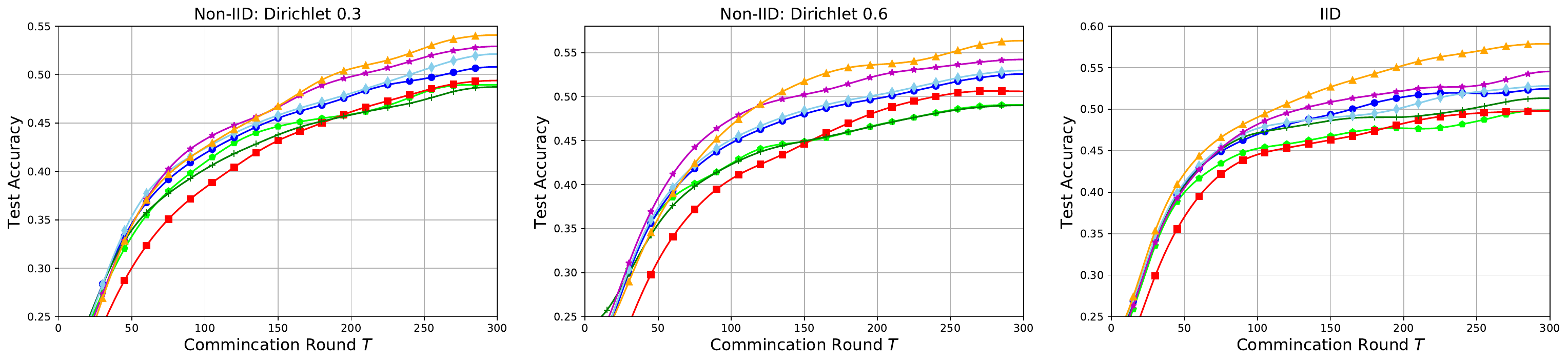}
        \caption{CIFAR-10}
        \label{fig:cifar10}
    \end{subfigure}
    \vspace{-0.55cm}
    \caption{\small The averaged testing accuracy on \textit{EMNIST} and \textit{CIFAR-10} dataset under symmetric noise for all compared methods. }
 \label{fig:all}
\vspace{-0.6cm}
\end{figure*}

\subsection{Experiment Evaluation}\label{eva}


\noindent
\textbf{Performance with compared baselines.}
In Table \ref{table:all_baselines} and Figure \ref{fig:all}, we evaluate DP-FedSAM and DP-FedSAM-$\topk_k$ on EMNIST and CIFAR-10 datasets in both settings compared with all baselines from DP-FedAvg to DP-FedAvg-$\blurs$. The baseline methods seem to be overfitting in Table \ref{table:all_baselines}, especially on more complex data (e.g., CIFAR-100 in Table \ref{table:cifar100_all_baselines}). The main reasons are the random noise introduced by DP and the over-fitting and inconsistency of the local models. From all these results, it is seen that our proposed algorithms outperform other baselines under symmetric noise both on accuracy and generalization perspectives.
It means that we significantly improve the performance and generate a better trade-off between performance and privacy in DPFL.
For instance, the averaged testing accuracy is $85.90\%$ in DP-FedSAM and $87.70\%$ in DP-FedSAM-$\topk_k$ on EMNIST in the IID setting, which are better than other baselines. Meanwhile, the difference between training accuracy and test accuracy is $9.71\%$ in DP-FedSAM and $8.40\%$ in DP-FedSAM-$\topk_k$, while that is $17.74\%$ in DP-FedAvg and $16.41\%$ in Fed-SMP-$\topk_k$, respectively. That means our algorithms significantly mitigate the performance degradation issue caused by DP.

\noindent
\textbf{Impact of Non-IID levels.}
In the experiments under different participation cases as shown in Table \ref{table:all_baselines}, we further prove the robust generalization of the proposed algorithms. Heterogeneous data distribution of local clients is set to various participation levels from IID, Dirichlet 0.6, and Dirichlet 0.3, which makes the training of the global model more difficult. On EMNIST, as the Non-IID level decreases, DP-FedSAM achieves better generalization than DP-FedAvg, and the difference between training accuracy and test accuracy in DP-FedSAM $\{19.47\%, 10.75\%, 9.71\%\}$ are lower than that in DP-FedAvg $\{26.18\%, 17.35\%, 17.74\%\}$. Similarly, the difference values in DP-FedSAM-$\topk_k$ $\{17.50\%, 11.07\%, 8.40\%\}$ are also lower than that in Fed-SMP-$\topk_k$ $\{23.56\%, 16.31\%, 16.41\%\}$. These observations confirm that our algorithms are more robust than baselines in various degrees of heterogeneous data.
\subsection{Discussion for DP with SAM in FL}\label{exper_DP}
In this subsection, we empirically discuss how SAM mitigates the negative impacts of DP from the aspects of the norm of local update and the visualization of loss landscape and contour. Meanwhile, we also observe the training performance with SAM under different privacy budgets $\epsilon$ compared with all baselines. These experiments are conducted on CIFAR-10 with ResNet-18 \cite{he2016deep} and Dirichlet $\alpha=0.6$.

\begin{figure}
\centering
\includegraphics[width=0.48\textwidth]{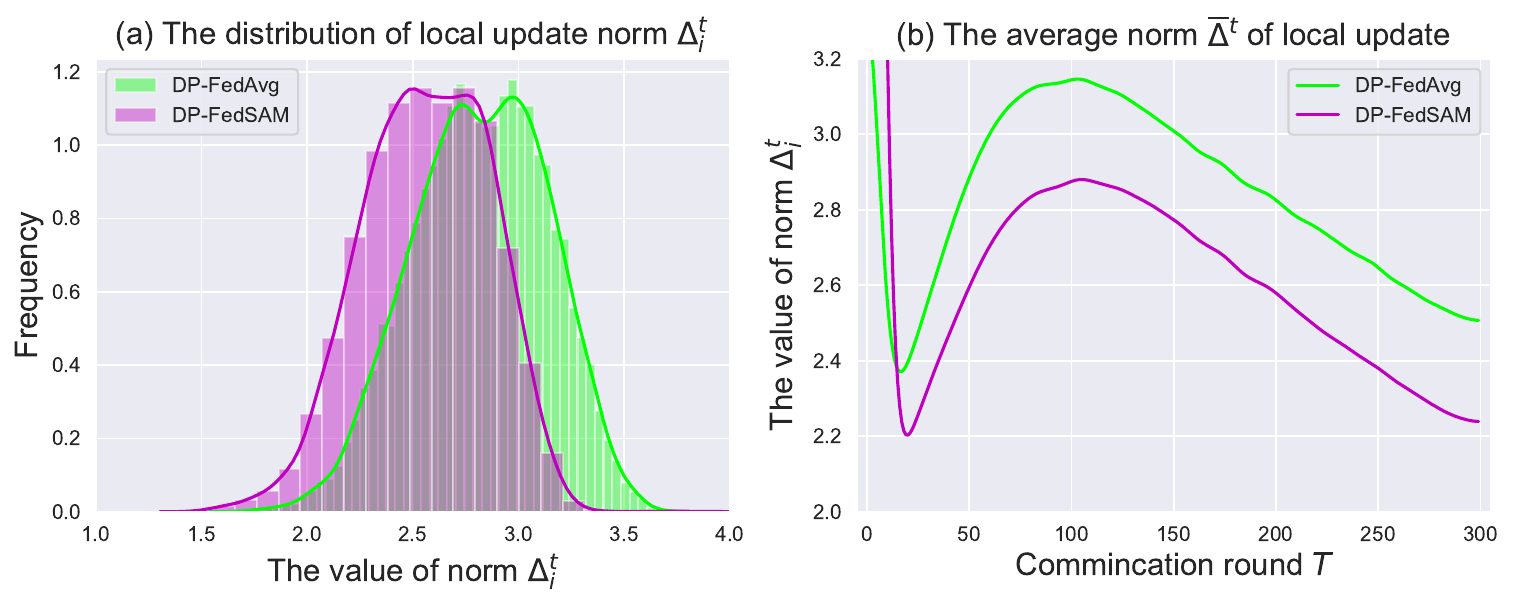}
\vspace{-0.4cm}
\caption{\small Norm distribution and average norm of local updates.
}
\vspace{-0.4cm}
\label{fig:norm}
\end{figure}
\begin{figure*}[ht]
\centering
\includegraphics[width=0.95\textwidth]{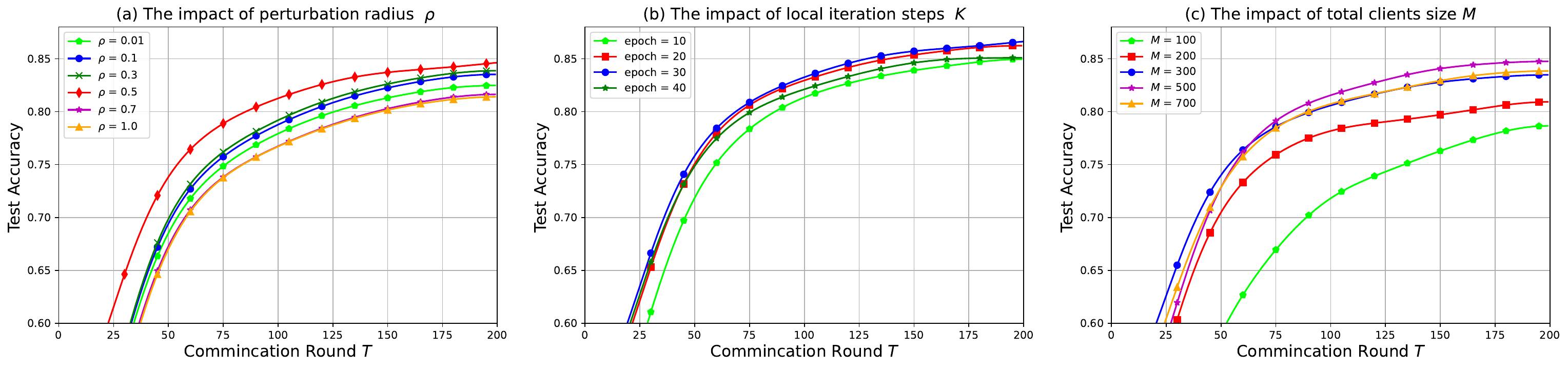}
\vspace{-0.35cm}
\caption{\small Impact of hyper-parameters: perturbation radius $\rho$, local iteration steps $K$, total clients size $M$.}
\vspace{-0.35cm}
\label{fig:abla}
\end{figure*}

\noindent
\textbf{The norm of local update.}
To validate the theoretical results for mitigating the adverse impacts of the norm of local updates, we conduct experiments on DP-FedSAM and DP-FedAvg with clipping threshold $C=0.2$ as shown in Figure \ref{fig:norm}. We show the norm $\Delta_i^t$ distribution and average norm $\overline{\Delta}^t$ of local updates before clipping during the communication rounds. In contrast to DP-FedAvg, most of the norm is distributed at the smaller value in our scheme in Figure \ref{fig:norm} (a), which means that the clipping operation drops less information. Meanwhile, the on-average norm $\overline{\Delta}^t$ is smaller than DP-FedAvg as shown in Figure \ref{fig:norm} (b). These observations are also consistent with our theoretical results in Section \ref{th}.

\begin{table}
\centering
\caption{Performance comparison under different privacy budgets.}
\vspace{-0.2cm}
\label{privacy}
\small 
\renewcommand{\arraystretch}{0.9}
\resizebox{\linewidth}{!}{
\begin{tabular}{ccccc} 
\toprule
\multirow{2}{*}{Algorithm} & \multicolumn{4}{c}{Averaged test accuracy (\%) under different privacy budgets $\epsilon$}                                                                                                            \\ 
\cmidrule{2-5}
                           &  $\epsilon$ = 4                      & $\epsilon$ = 6                                  &$\epsilon$ =  8                                  & $\epsilon$ = 10                                  \\ 
\midrule
DP-FedAvg                  & 38.23 $\pm$
  0.15 & 43.87 $\pm$ 0.62 & 46.74 $\pm$ 0.03 & 49.06 $\pm$ 0.49  \\
Fed-SMP-$\randk_k$              & 33.78 $\pm$ 0.92    & 42.21 $\pm$ 0.21 & 48.20 $\pm$ 0.05 & 50.62 $\pm$ 0.14  \\
Fed-SMP-$\topk_k$               & 38.99 $\pm$0.50     & 46.24 $\pm$ 0.80 & 49.78 $\pm$ 0.78 & 52.51 $\pm$ 0.83  \\
DP-FedAvg-$\blur $            & 38.23 $\pm$ 0.70    & 43.93 $\pm$ 0.48 & 46.74 $\pm$ 0.92 & 49.06 $\pm$ 0.13  \\
DP-FedAvg-$\blurs $             & 39.39 $\pm$0.43     & 46.64 $\pm$ 0.36  & 50.18 $\pm$ 0.27 & 52.91 $\pm$ 0.57  \\
DP-FedSAM                  & \textbf{39.89$\pm$ 0.17}   & 47.92 $\pm$ 0.23 & 51.30 $\pm$ 0.95 & 53.18 $\pm$ 0.40  \\
DP-FedSAM-$\topk_k$              & 38.96 $\pm$ 0.61    & \textbf{49.17 $\pm$ 0.15} & \textbf{53.64 $\pm$ 0.12} & \textbf{56.36 $\pm$ 0.36}  \\
\bottomrule
\end{tabular}
}
\vspace{-0.45cm}
\end{table}

\noindent
\textbf{Performance under different privacy budgets $\epsilon$.}
Table \ref{privacy} shows the test accuracies for the different level privacy guarantees. Note that $\epsilon$ is not a hyper-parameter, but can be obtained by the privacy design in each round such as Eq. (\ref{eq:accumulative_eps}).
Our methods consistently outperform the previous SOTA methods under various privacy budgets $\epsilon$.
Specifically, DP-FedSAM and DP-FedSAM-$\topk_k$ significantly improve the accuracy of DP-FedAvg and Fed-SMP-$\topk_k$ by $1\% \sim 4\%$ and $3\% \sim 4\%$ under the same $\epsilon$, respectively. 
Furthermore, the test accuracy can improve as the privacy budget $\epsilon$ increases, which suggests us a better balance is necessary between training performance and privacy. There is a better trade-off when achieving better accuracy under the same $\epsilon$ or a smaller $\epsilon$ under approximately the same accuracy.

\noindent
\textbf{Loss landscape and contour.}
The discussion and visualizing results are placed in \textbf{Appendix} \ref{exper_DP_appendix} due to limited space.

\subsection{Ablation Study}
We verify the influence of hyper-parameters in DP-FedSAM on EMNIST with Dirichlet partition  $\alpha=0.6$. 

\noindent
\textbf{Perturbation weight $\rho$.}
Perturbation weight $\rho$ has an impact on performance as the adding perturbation is accumulated when the communication round $T$ increases. 
To select a proper value for our algorithms, we conduct some experiments on various perturbation radius from the set $\{ 0.01, 0.1, 0.3, 0.5, 0.7, 1.0\}$ in Figure \ref{fig:abla} (a). As $\rho = 0.5$, we achieve better convergence and performance. 

\noindent
\textbf{Local iteration steps $K$.}
Large local iteration steps $K$ can help the convergence in previous DPFL work \cite{cheng2022differentially} with the theoretical guarantees. To investigate the acceleration on $T$ by adopting a larger $K$, we fix the total batchsize and change local training epochs. In Figure \ref{fig:abla} (b), our algorithm can accelerate the convergence in Theorem \ref{th:conver} as a larger $K$ is adopted, that is, use a larger epoch value. However, the adverse impact of clipping on training increases as $K$ is too large, e.g., epoch 
= $40$. Thus we choose $30$ local epoch.

\noindent
\textbf{Client size $M$.}
We compare the performance between different numbers of client participation $m=\{ 100, 200, 300, 500, 700\}$ with the same hyper-parameters in Figure \ref{fig:abla} (c). Compared with larger $m=500$, the smaller $m$ may have worse performance due to large variance $\sigma^2 C^2/m$ in DP noise with the same setting. Meanwhile, when $m$ is too large such as $M=700$, the performance may degrade as the local data size decreases. 
\begin{table}
\caption{The averaged training accuracy and testing accuracy.}
\vspace{-0.2cm}
\label{ab:sam}
\small 
\centering
\renewcommand{\arraystretch}{0.9}
\resizebox{\linewidth}{!}{
\begin{tabular}{cccc}
\toprule
Algorithm  & Train  (\%)& Validation (\%) & Differential value (\%)  \\
\midrule
DP-FedAvg  &    99.55$\pm$0.02   &       82.20$\pm$ 0.35        &   17.35$\pm$0.32              \\
DP-FedSAM  &    95.07$\pm$0.45   &    84.32$\pm$0.19 $\uparrow$        &     10.75$\pm$0.26              $\downarrow$  \\
Fed-SMP-$\topk_k$   &   99.72$\pm$0.02    &     83.41  $\pm$ 0.91     &        16.31$\pm$ 0.89            \\
DP-FedSAM-$\topk_k$ & 95.87$\pm$0.52        &     84.80$\pm$0.60  $\uparrow$      & 11.07$\pm$0.08  $\downarrow$  \\   
\bottomrule
\end{tabular}}
\vspace{-0.4cm}
\end{table}

\noindent
\textbf{Effect of SAM.}
As shown in Table \ref{ab:sam}, it is seen that DP-FedSAM and DP-FedSAM-$\topk_k$ can achieve performance improvement and better generalization compared with DP-FedAvg and Fed-SMP-$\topk_k$ as SAM optimizer is adopted with the same setting, respectively.

\section{Conclusion}
In this paper, we focus on severe performance degradation issue caused by dropped model information and exacerbated model inconsistency. And we are the first to alleviate this issue from the optimizer perspective and propose a novel and effective framework DP-FedSAM with a flatter loss landscape. Meanwhile, we present the analysis in detail of how SAM mitigates the adverse impacts of DP and achieve a 
tighter bound on convergence. 
Moreover, it is the first analysis to combine the impacts of the on-average norm of local updates and local update consistency among clients on training, and simultaneously provide the experimental observations.
Finally, empirical results also verify the superiority of our approach on several real-world data.

\vspace{0.1cm}

\noindent
\textbf{Acknowledgement} 
This work is supported by STI 2030—Major Projects (No. 2021ZD0201405), Shenzhen Philosophy and Social Science Foundation (Grant No. SZ2021B005).

{
\small
\bibliographystyle{ieee_fullname}
\bibliography{ref}
}

\clearpage
\newpage
\onecolumn 

\vspace{0.5in}
\begin{center}
 \rule{6.875in}{0.7pt}\\ 
 {\Large\bf Supplementary Material for\\ `` Make Landscape Flatter in Differentially Private Federated Learning ''}
 \rule{6.875in}{0.7pt}
\end{center}
\appendix

\section{More Implementation Detail}\label{imp_detail}
\subsection{Dataset}
EMNIST \cite{cohen2017emnist} is a 62-class image classification dataset. In this paper, we use 20\% of the dataset, which includes 88,800 training samples and 14,800 validation examples. 
Both CIFAR-10 and CIFAR-100 \cite{krizhevsky2009learning} have 60,000 images. In addition, these images are divided into 50,000 training samples and 10,000 validation examples. CIFAR-100 has finer labeling, with 100 unique labels, in comparison to CIFAR-10, having 10 unique labels. Furthermore, we
divide these datasets to each client based on Dirichlet allocation over 500 clients by default.

\subsection{Configuration}
For the EMNIST dataset, we set the mini-batch size to 32 and train with a simple CNN model, which includes two convolutional layers with 5×5 kernels, max pooling, followed by a 512-unit dense layer. For CIFAR-10 and CIFAR-100 datasets, we set the mini-batch size to 50 and train with ResNet-18 \cite{he2016deep} architecture. 
For each algorithm and each dataset, the learning rate is set via grid search on the set $\{10^{-0.5}, 10^{-1}, 10^{-1.5}, 10^{-2}\}$. The weight perturbation ratio $\rho$ is set via grid search on the set $\{0.01, 0.1, 0.3, 0.5, 0.7, 1.0\}$. For all methods using the sparsification technique, the sparsity ratio is set to $p=0.4$.

\section{Additional Experiment on CIFAR-100}


\begin{figure*}[ht]
\centering
\includegraphics[width=1\textwidth]{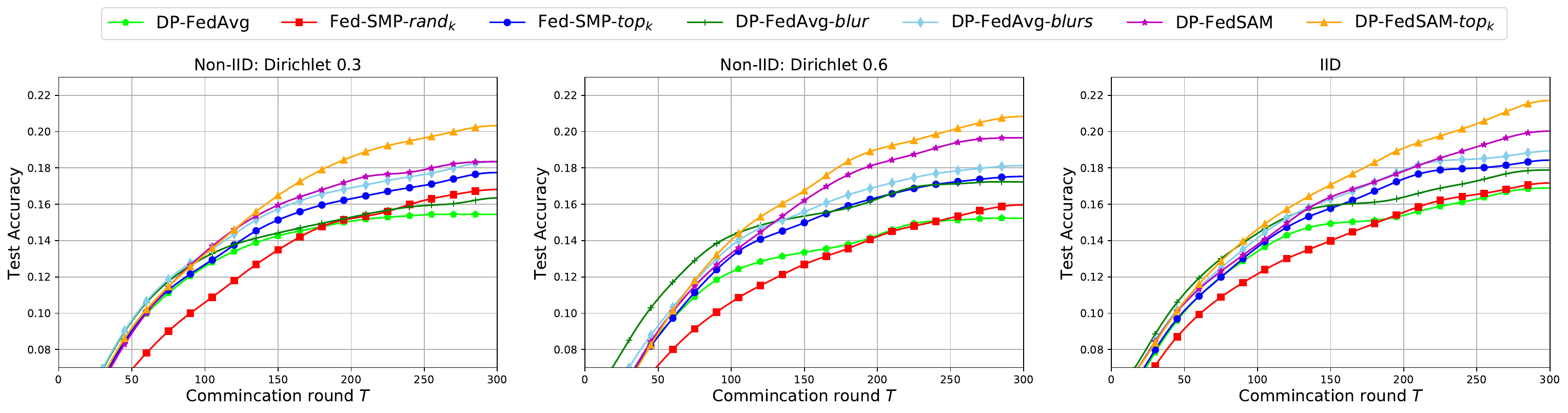}
\caption{\small The averaged testing accuracy on \textit{CIFAR-100} dataset under symmetric noise for all compared methods.}
\label{fig:cifar100}
\end{figure*}

\begin{table}[ht]
\centering
\caption{Averaged training  and testing accuracy (\%) on \textit{CIFAR-100} in both IID and Non-IID settings under symmetric noise for all compared methods. Note that the performance of the CIFAR-100 dataset is relatively poor across all algorithms due to the more severe impact of DP in complex tasks.}
\small
\renewcommand{\arraystretch}{0.5}
\label{table:cifar100_all_baselines}
\resizebox{0.95\linewidth}{!}{
\begin{tabular}{ccccccc} 
\toprule
\multirow{2}{*}{Algorithm} & \multicolumn{2}{c}{Dirichlet~0.3} & \multicolumn{2}{c}{Dirichlet~0.6} & \multicolumn{2}{c}{IID}  \\ 
\cmidrule{2-7}
                           & Train      & Validation           & Train      & Validation           & Train      & Validation            \\ 
\midrule
DP-FedAvg                  & 91.14$\pm$0.16 & 16.10$\pm$0.71           & 92.33$\pm$0.08 & 15.92$\pm$0.39           & 94.01$\pm$0.10 & 17.47$\pm$0.47            \\
Fed-SMP-$\randk_k$              & 90.70$\pm$0.01 & 17.25$\pm$0.16           & 92.28$\pm$0.32 & 17.50$\pm$0.19           & 94.31$\pm$0.02 & 17.68$\pm$0.44            \\
Fed-SMP-$\topk_k$                & 92.58$\pm$0.24 & 18.58$\pm$0.25           & 93.51$\pm$0.11 & 18.07$\pm$0.09           & 95.06$\pm$0.05 & 19.09$\pm$0.56            \\
DP-FedAvg-$\blur$             & 91.27$\pm$0.01 & 17.03$\pm$0.09           & 92.33$\pm$0.03 & 17.92$\pm$0.01           & 94.01$\pm$0.04 & 18.47$\pm$0.02            \\
DP-FedAvg-$\blurs$            & 92.98$\pm$0.24 & 18.98$\pm$0.25           & 94.01$\pm$0.11 & 18.27$\pm$0.19           & 95.46$\pm$0.05 & 19.59$\pm$0.06            \\
DP-FedSAM                  & 82.19$\pm$0.01 & 18.88$\pm$0.31           & 85.47$\pm$0.13 & 19.09$\pm$0.15           & 87.12$\pm$0.37 & 20.64$\pm$0.48            \\
DP-FedSAM-$\topk_k$              & 84.49$\pm$0.24 & \textbf{20.85$\pm$0.63}  & 88.23$\pm$0.23 & \textbf{21.24$\pm$0.69}  & 89.86$\pm$0.21 & \textbf{22.30$\pm$0.05}   \\
\bottomrule
\end{tabular}}
\end{table}

\subsection{Performance with Compared Baselines}

In Table \ref{table:cifar100_all_baselines} and Figure \ref{fig:cifar100}, we evaluate DP-FedSAM and DP-FedSAM-$\topk_k$ on CIFAR-100 dataset in both settings compared with all baselines from DP-FedAvg to DP-FedAvg-$\blurs$. From all these results, it is clearly seen that our proposed algorithms outperform other baselines under symmetric noise both on accuracy and generalization perspectives. It means that we significantly improve the performance and generate a better trade-off between performance and privacy in DPFL. 
For instance, in the IID setting, the averaged testing accuracy is $20.64\%$ in DP-FedSAM, where the accuracy gain is $3.17\%$ compared with DP-FedAvg.
And the average testing accuracy is $22.30\%$ in DP-FedSAM-$\topk_k$, 
where the accuracy gain is $3.21\%$ compared with Fed-SMP-$\topk_k$.
That means our algorithms significantly mitigate the performance degradation issue caused by DP.


\subsection{Impact of Non-IID levels}
Under different participation cases as shown in Table \ref{table:cifar100_all_baselines}, we further prove the robust generalization of the proposed algorithms. Heterogeneous data distribution of local clients is set to various participation levels from IID, Dirichlet 0.6, and Dirichlet 0.3, which makes the training of the global model more difficult. 
For instance, compared with DP-FedAvg on CIFAR-100, the test accuracy gain in DP-FedSAM is $\{2.78\%, 3.17 \%, 3.17\%\}$. Meanwhile, the test accuracy gain in DP-FedSAM-$\topk_k$ is $\{2.27\%, 3.17\%, 3.21\%\}$ compared with Fed-SMP-$\topk_k$.
These observations confirm that our algorithms are more robust than baselines in various degrees of heterogeneous data.


\section{More details on Discussion for DP with SAM in FL}\label{exper_DP_appendix}
\begin{figure}
\centering
\begin{subfigure}{0.57\linewidth}
\centering
\includegraphics[width=1\textwidth]{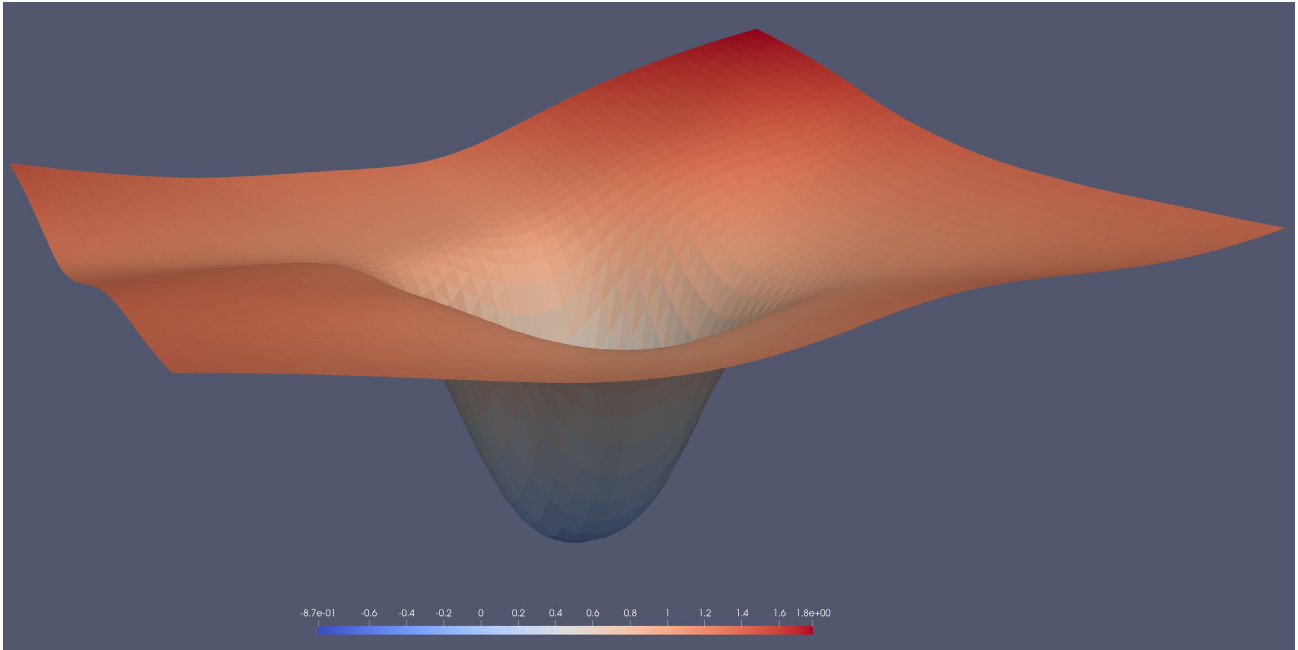}
\caption{\small Loss landscape}
\end{subfigure}
\hfill
\begin{subfigure}{0.41\linewidth}
\centering
\includegraphics[width=1\textwidth]{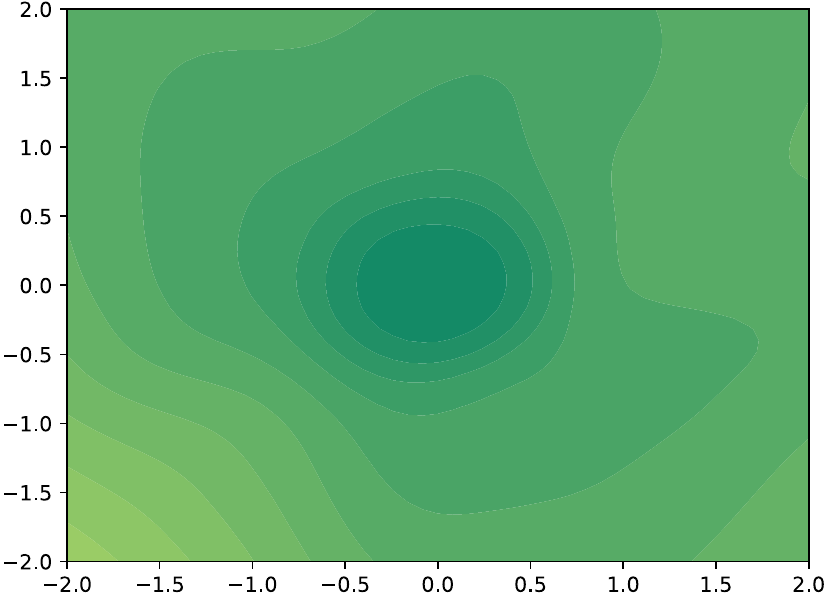}
\caption{\small Loss surface contour}
\end{subfigure}
\vspace{-0.2cm}
\caption{\small Loss landscape and surface contour of DP-FedSAM. Compared with DP-FedAvg in the left of Figure \ref{landscape_fedavg_dpfedavg} (a) and (b) with the same setting, DP-FedSAM has a flatter landscape with both better generalization ability (flat minima, see Figure \ref{sam_land} (a)) and weight perturbation robustness (see Figure \ref{sam_land} (b)).
}
\vspace{-0.2cm}
\label{sam_land}
\end{figure}

\noindent
\textbf{Loss landscape and contour.}
To visualize the sharpness of the flat minima and observe robustness to DP noise obtained by DP-FedSAM, we show the loss landscape and surface contour following by the plotting method \cite{li2018visualizing} in Figure \ref{sam_land}. It is clear that DP-FedSAM has flatter minima and better robustness to DP noise than DP-FedAvg in the left of Figure \ref{landscape_fedavg_dpfedavg} (a) and (b), respectively. It indicates that our proposed algorithm achieves better generalization and makes the training process more adaptive to the DPFL setting.

\section{Discussion for DP Guarantee in DP-FedSAM with Sparsification}\label{DP-sam-topk}
Sparsification is a very common method when considering privacy protection to introduce a large amount of random noise in FL \cite{hu2022federated, cheng2022differentially, Hu2021federated}. It retains only the larger weight part of each layer of the local model with a sparsity ratio of $k/d$ ($d$ is the weight scale), and the rest are sparse. The advantage is that the amount of random noise can be reduced (no noise needs to be added to the sparse weight position), so the performance can be improved, which has been thoroughly verified in \cite{hu2022federated, cheng2022differentially, Hu2021federated}. In our methods, SAM needs to perform two gradient calculations and sparsification may lead to some performance degradation because the model is compressed and some information may be lost. 

Existing work \cite{hu2022federated} has verified SGD and top-k sparsification satisfying the Renyi DP. SAM optimizer only adds perturbation on the basis of SGD and affects the model during training. And both SAM and top k sparsification are performed before the DP process, thereby satisfying the Renyi DP.
\section{Main Proof} \label{appendix_th}
\subsection{Preliminary Lemmas}
\begin{lemma}\label{e_delta}
(Lemma B.1, \cite{Qu2022Generalized}) Under Assumptions~\ref{a1}-\ref{a2}, the updates for any learning rate satisfying $\eta \leq \frac{1}{4KL}$ have the drift due to $\delta_{i,k} - \delta$:
\begin{equation}
    \frac{1}{M}\sum_{i}\mathbb{E} [\|\delta_{i,k} - \delta \|^2 ] \leq 2K^2 L^2 \eta^2 \rho^2 . \nonumber
\end{equation}
Where 
\begin{equation}
		\delta = \rho \frac{\nabla f(\mathbf{w}^t)}{\|\nabla f(\mathbf{w}^t)\|}, ~~~ \delta_{i,k} = \rho \frac{\nabla F_i (\mathbf{w}^{t,k} ,\xi_i )}{\|\nabla F_i (\mathbf{w}^{t,k}, \xi_i )\|}. \nonumber
\end{equation}
\end{lemma}
\begin{lemma}\label{e_w}
(lemma B.2, \cite{Qu2022Generalized}) Under above assumptions, the updates for any learning rate satisfying $\eta_l \leq \frac{1}{10KL}$ have the drift due to $\mathbf{w}^{t,k}(i)  - \mathbf{w}^t$:
\begin{equation}
\begin{split}
    \frac{1}{M}\sum_{i}\mathbb{E} [\|\mathbf{w}^{t,k}(i) - \mathbf{w}^t \|^2 ] 
    & \leq 5K\eta^2 \Big(2L^2 \rho^2 \sigma_l^2
    + 6K(3\sigma_g^2 + 6L^2 \rho^2 )  + 6K\|\nabla f(\mathbf{w}^t)\|^2 \Big) + 24K^3 \eta^4 L^4 \rho^2 . \nonumber
\end{split}
\end{equation}
\end{lemma}

\begin{lemma}\label{y_k-x_k}
The two model parameters conducted by two adjacent datasets which differ only one sample from client $i$ in the communication round $t$, 
\begin{equation}
    \sum_{k=0}^{K-1}\|\mathbf{y}^{t,k}(i)- \mathbf{x}^{t,k}(i) \|_2^2
    \leq 2K \max \|\Delta_i^t(\mathbf{y})-\Delta_i^t(\mathbf{x})\|_2^2 \nonumber.
\end{equation}
\end{lemma}
\begin{proof}
Recall the local update from client $i$ is $\sum_{k=0}^{K-1}\mathbf{w}^{t,k}(i) = \sum_{k=0}^{K-1}\mathbf{w}^{t,k-1}(i) + \Delta_i^t $, (the initial value is assumed as $ \mathbf{w}^{t, -1}= \mathbf{w}^{t, 0} =\mathbf{w}^{t}$). Then,
\begin{equation}
    \begin{split}
      & \sum_{k=0}^{K-1}\|\mathbf{y}^{t,k}(i)- \mathbf{x}^{t,k}(i) \|_2^2
    \leq 2 \sum_{k=0}^{K-1}\|\mathbf{y}^{t,k-1}(i)- \mathbf{x}^{t,k-1}(i) \|_2^2\\
    &+ 2 \|\Delta_i^t(\mathbf{y})-\Delta_i^t(\mathbf{x})\|_2^2. \nonumber
    \end{split}
\end{equation}
The recursion from $\tau=0$ to $k$ yields
\begin{equation}
    \sum_{k=0}^{K-1}\|\mathbf{y}^{t,k}(i)- \mathbf{x}^{t,k}(i) \|_2^2
    \overset{a)}{\leq} 2K \max \|\Delta_i^t(\mathbf{y})-\Delta_i^t(\mathbf{x})\|_2^2 \nonumber.
\end{equation}
Where a) uses the initial value $\mathbf{w}^t(i)=\mathbf{x}^{t, 0}(i) =\mathbf{y}^{t, 0}(i)$ and $0<k \leq K$.
\end{proof}

\begin{lemma}\label{Delta_average}
Under assumption \ref{a1} and \ref{a3},
the average of local update after the clipping operation from selected clients is 
\begin{equation}
    \mathbb{E}\|\frac{1}{m}\sum_{i\in \mathcal{W}^t}\tilde{\Delta}_i^t\|^2 \leq 3K\eta^2(L^2\rho^2+B^2) \nonumber
\end{equation}
\end{lemma}
\begin{proof}
\begin{equation}
\begin{split}
    \mathbb{E}\|\frac{1}{m}\sum_{i\in \mathcal{W}^t}\tilde{\Delta}_i^t\|^2 
    & \leq \mathbb{E} \|\frac{1}{m}\sum_{i\in \mathcal{W}^t}\sum_{i=0}^{K-1}\eta \tilde{\mathbf{g}}^{t,k}(i) \cdot \alpha_i^t\|^2  \leq \frac{\eta^2}{m} \sum_{i\in \mathcal{W}^t}\sum_{i=0}^{K-1}\mathbb{E} \| \nabla F_i(\mathbf{w}^{t,k}(i) +\delta; \xi_i) - \nabla F_i(\mathbf{w}^{t,K}(i); \xi_i)\\
    & + \nabla F_i(\mathbf{w}^{t,k}(i); \xi_i) - \nabla F_i(\mathbf{w}^{t}(i)) + \nabla F_i(\mathbf{w}^{t}(i))\|^2\\
    & \overset{a)}{\leq} 3K\eta^2(L^2\rho^2+B^2) ,\nonumber
\end{split}
\end{equation}
where a) uses assumption \ref{a1} and \ref{a3} and 
\begin{equation}
    \alpha_i^t := \min \Big (1, \frac{C}{\eta \|\sum_{k=0}^{K-1}\tilde{\mathbf{g}}^{t,k}(i)\|} \Big).\nonumber
\end{equation}
\end{proof}


\subsection{Proof of Sensitivity Analysis}
\begin{proof}[Proof of Theorem \ref{th:sensitivity}]
Recall that the local update before clipping and adding noise on client $i$ is $\Delta_i^t = \mathbf{w}^{t,K}(i)-\mathbf{w}^{t,0}(i)$. Then,
\begin{equation}
 \begin{split}
    \mathbb{E} \mathcal{S}_{\Delta_i^t}^2 & = \max  \mathbb{E} \| \Delta_i^t(\mathbf{x}) - \Delta_i^t(\mathbf{y})\|_2^2\\
    & = \mathbb{E} \| \mathbf{x}^{t,K}(i) - \mathbf{x}^{t,0}(i) - (\mathbf{y}^{t,K}(i) - \mathbf{y}^{t,0}(i))\|_2^2\\
    & =  \eta^2\mathbb{E} \sum_{k=0}^{K-1}\|\nabla F_i(\mathbf{x}^{t,k}(i)+\delta_x; \xi_i) -\nabla F_i(\mathbf{y}^{t,k}(i)+\delta_y; \xi_i^{'})\|_2^2 \\
    & =  \eta^2L^2 \mathbb{E} \sum_{k=0}^{K-1}\|\mathbf{y}^{t,k}(i)- \mathbf{x}^{t,k}(i) + (\delta_y-\delta_x)\|_2^2  \\
    & \overset{a)}{\leq} 2\eta^2L^2 K \max \|\Delta_i^t(\mathbf{y})-\Delta_i^t(\mathbf{x})\|_2^2 
    + 2\eta^2L^2\rho^2 \mathbb{E} \sum_{k=0}^{K-1} 
    \Big \|\frac{\nabla
    F_i(\mathbf{y}^{t,k}(i)+\delta_y; \xi_i^{'})}{\|\nabla
    F_i(\mathbf{y}^{t,k}(i)+\delta_y; \xi_i^{'})\|_2}- \frac{\nabla
    F_i(\mathbf{y}^{t,k}(i); \xi_i^{'})}{\|\nabla
    F_i(\mathbf{y}^{t,k}(i); \xi_i^{'})\|_2}\\
    & + ( \frac{\nabla
    F_i(\mathbf{x}^{t,k}(i); \xi_i)}{\|\nabla
    F_i(\mathbf{x}^{t}(i,k); \xi_i)\|_2} - \frac{\nabla
    F_i(\mathbf{x}^{t,k}(i)+\delta_x; \xi_i)}{\|\nabla
    F_i(\mathbf{x}^{t,k}(i)+\delta_x; \xi_i)\|_2})
     + \frac{\nabla
    F_i(\mathbf{y}^{t,k}(i); \xi_i^{'}))}{\|\nabla
    F_i(\mathbf{y}^{t,k}(i);\xi_i^{'}))\|_2} - \frac{\nabla
    F_i(\mathbf{x}^{t,k}(i);\xi_i)}{\|\nabla
    F_i(\mathbf{x}^{t,k}(i); \xi_i)\|_2}\Big \|_2^2\\
    & \leq 2\eta^2L^2K \max \|\Delta_i^t(\mathbf{y})-\Delta_i^t(\mathbf{x})\|_2^2 +
     6\eta^2\rho^2L^2 \mathbb{E} \sum_{k=0}^{K-1}\Big(4  + \frac{1}{\rho^2}\Big\|\rho \frac{\nabla
    F_i(\mathbf{y}^{t,k}(i); \xi_i^{'}))}{\|\nabla
    F_i(\mathbf{y}^{t,k}(i);\xi_i^{'}))\|_2} - \rho \frac{\nabla
    f(\mathbf{y}^{t})}{\|\nabla
    f(\mathbf{y}^{t})\|_2} \\
    &+ (  \rho \frac{\nabla
    f(\mathbf{x}^{t})}{\|\nabla
    f(\mathbf{x}^{t})\|_2} - \rho \frac{\nabla
    F_i(\mathbf{x}^{t,k}(i);\xi_i)}{\|\nabla
    F_i(\mathbf{x}^{t,k}(i); \xi_i)\|_2}) 
    + \rho \frac{\nabla
    f(\mathbf{y}^{t})}{\|\nabla
    f(\mathbf{y}^{t})\|_2} - \rho \frac{\nabla
    f(\mathbf{x}^{t})}{\|\nabla
    f(\mathbf{x}^{t})\|_2}\Big\|_2^2 \Big )\\
    & \overset{b)}{\leq}
    2\eta^2L^2 K \mathcal{S}_{\Delta_i^t}^2 + 6\eta^2\rho^2KL^2(4  + 12K^2L^2\eta^2+ 6) \\
    & \leq \frac{6\eta^2\rho^2KL^2(12K^2L^2\eta^2+ 10)}{1-2\eta^2L^2 K}
\end{split}   
\end{equation}
where a) and b)  uses lemma \ref{y_k-x_k} and \ref{e_delta}, respectively. 

When the local adaptive learning rate satisfies $\eta=\mathcal{O}({1}/{L\sqrt{KT}})$ and the perturbation amplitude $\rho$
proportional to the learning rate, e.g., $\rho = \mathcal{O}(\frac{1}{\sqrt{T}})$, we have
\begin{align}
\small
    \mathbb{E}\mathcal{S}^2_{\Delta_i^t} \leq
    \mathcal{O}\left(\frac{1}{T^2}\right). 
\end{align}
\end{proof}

For comparison, we also present the expected squared sensitivity of local update with SGD in DPFL as follows. It is clearly seen that the upper bound in $  \mathbb{E}\mathcal{S}^2_{\Delta_i^t, SAM}$ is tighter than that in $\mathbb{E}\mathcal{S}^2_{\Delta_i^t, SGD}$.


\begin{proof}[Proof of sensitivity with SGD in FL.]
\begin{equation}
    \begin{split}
         \mathbb{E}  \mathcal{S}_{\Delta_i^t, SGD}^2 & = \max  \mathbb{E}  \| \Delta_i^t(\mathbf{x}) - \Delta_i^t(\mathbf{y})\|_2^2
     = \eta^2 \mathbb{E} \sum_{i=0}^{K-1}\|\nabla F_i( \mathbf{x}^{t,k}(i); \xi_i) - \nabla F_i( \mathbf{y}^{t,k}(i);\xi_i^{'})\|_2^2\\
    & = \eta^2 \mathbb{E} \sum_{i=0}^{K-1}\|\nabla F_i( \mathbf{x}^{t,k}(i); \xi_i) -\nabla F_i( \mathbf{x}^{t}(i)) + \nabla F_i( \mathbf{x}^{t}(i)) 
    - \nabla F_i( \mathbf{y}^{t}(i))  +\nabla F_i( \mathbf{y}^{t}(i)) -
    \nabla F_i( \mathbf{y}^{t,k}(i);\xi_i^{'})\|_2^2\\
    & \overset{a)}{\leq}
    3\eta^2 \mathbb{E} \sum_{i=0}^{K-1}(2\sigma_l^2+L^2\|y^{t,k}(i)-x^{t,k}(i)\|_2^2)\\
    & \overset{b)}{\leq} 6\eta^2K\sigma_l^2+3\eta^2L^2K  \max  \mathbb{E} \| \Delta_i^t(\mathbf{x}) - \Delta_i^t(\mathbf{y})\|_2^2\\
    & \leq \frac{6\eta^2\sigma_l^2K}{1-3\eta^2KL^2}.
    \end{split}
\end{equation}
Where a) and b) uses assumptions \ref{a1}-\ref{a2} and lemma \ref{y_k-x_k}, respectively. Thus $\mathbb{E}\mathcal{S}^2_{\Delta_i^t, SGD} \leq \mathcal{O}(\frac{\sigma_l^2}{KL^2T})$ when $\eta=\mathcal{O}({1}/{L\sqrt{KT}})$.
\end{proof}

\subsection{Proof of Convergence Analysis}
\begin{proof}[Proof of Theorem \ref{th:conver}]
We define the following notations for convenience:
\begin{equation}
    \begin{split}
        & \tilde{\Delta}_i^t = -\eta\sum_{k=0}^{K-1}\tilde{\mathbf{g}}^{t,k}(i) \cdot \alpha_i^t;\\
        & \overline{\Delta_i^t} = -\eta\sum_{k=0}^{K-1}\tilde{\mathbf{g}}^{t,k}(i) \cdot \overline{\alpha}^t, \nonumber
    \end{split}
\end{equation}
where 
\begin{equation}
    \begin{split}
        & \alpha_i^t := \min \Big (1, \frac{C}{\eta \|\sum_{k=0}^{K-1}\tilde{\mathbf{g}}^{t,k}(i)\|} \Big), \\
        & \overline{\alpha}^t := \frac{1}{M}\sum_{i=1}^{M} \alpha_i^t,\\
        & \tilde{\alpha}^t :=\frac{1}{M}\sum_{i=1}^{M} |\alpha_i^t - \overline{\alpha}^t|.
    \end{split} \nonumber
\end{equation}
The Lipschitz continuity of $\nabla f$:
\begin{equation}
    \begin{split}
        & \mathbb{E} f(\mathbf{w}^{t+1}) \\
        & \leq \mathbb{E} f(\mathbf{w}^t) + \mathbb{E} \Big \langle\nabla f(\mathbf{w}^{t}), \mathbf{w}^{t+1}-\mathbf{w}^t \Big \rangle
        + \mathbb{E} \frac{L}{2}\|\mathbf{w}^{t+1}-\mathbf{w}^t\|^2\\
        & = \mathbb{E} f(\mathbf{w}^t) + \mathbb{E}\Big  \langle \nabla f(\mathbf{w}^{t}), \frac{1}{m}\sum_{i\in \mathcal{W}^t}\tilde{\Delta}_i^t + z_i^t \Big \rangle
         + \frac{L}{2}\mathbb{E}\Big  \|\frac{1}{m}\sum_{i\in \mathcal{W}^t}\tilde{\Delta}_i^t + z_i^t \Big \|^2\\
        & = \mathbb{E} f(\mathbf{w}^t) + 
         \underbrace{\Big \langle \nabla f(\mathbf{w}^{t}), \mathbb{E} \frac{1}{m}\sum_{i\in \mathcal{W}^t}\tilde{\Delta}_i^t \Big \rangle}_{\text{I}}
        + 
        \frac{L}{2} \mathbb{E} \underbrace{\Big \langle \|\frac{1}{m}\sum_{i\in \mathcal{W}^t}\tilde{\Delta}_i^t\|^2\Big \rangle}_{\text{II}} + \frac{L\sigma^2C^2d}{2m^2} ,
    \end{split}
\end{equation}
where $d$ represents dimension of $\mathbf{w}_i^{t,k}$ and the mean of noise $z_i^t$ is zero. Then, we analyze I and II, respectively.\\
For I, we have
\begin{equation}
    \begin{split}
        & \Big \langle\nabla f(\mathbf{w}^{t}), \mathbb{E} \frac{1}{m}\sum_{i\in \mathcal{W}^t}\tilde{\Delta}_i^t \Big  \rangle = \Big \langle\nabla f(\mathbf{w}^{t}), \mathbb{E}\frac{1}{M}\sum_{i=1}^M\tilde{\Delta}_i^t-\overline{\Delta}_i^t  \Big \rangle
        + \Big \langle\nabla f(\mathbf{w}^{t}), \mathbb{E}\frac{1}{M}\sum_{i=1}^M \overline{\Delta}_i^t  \Big \rangle.
    \end{split}
\end{equation}
Then we bound the two terms in the above equality, respectively. For the first term, we have
\begin{equation}
    \begin{split}
        &\mathbb{E} \Big \langle\nabla f(\mathbf{w}^{t}), \mathbb{E}\frac{1}{M}\sum_{i=1}^M\tilde{\Delta}_i^t-\overline{\Delta}_i^t  \Big \rangle\\
        & \leq\mathbb{E} \Big \langle\nabla f(\mathbf{w}^{t}), \mathbb{E}\frac{1}{M}\sum_{i=1}^M \sum_{k=0}^{K-1}\eta |\alpha_i^t - \overline{\alpha}^t|\tilde{\mathbf{g}}^{t,k}(i)\Big \rangle\\
        & \leq \frac{\eta K}{M}\sum_{i=1}^{M}\mathbb{E}|\alpha_i^t - \overline{\alpha}^t| \Big \langle \nabla F_i(\mathbf{w}^{t}),\tilde{\mathbf{g}}^{t,k}(i) \Big \rangle\\
        & \overset{a)}{\leq} \frac{\eta K}{M}\sum_{i=1}^{M}\mathbb{E}|\alpha_i^t - \overline{\alpha}^t| \Big(-\frac{1}{2}(\|\nabla F_i(\mathbf{w}^{t,k})\|^2+\|F_i(\mathbf{w}^{t,k}+\delta; \xi_i)\|^2)
        + \frac{1}{2} \|\nabla F_i(\mathbf{w}^{t,k}+\delta; \xi_i) - \nabla F_i(\mathbf{w}^{t,k}; \xi_i)\|^2\Big)\\
        & \overset{b)}{\leq}\eta \tilde{\alpha}^t K(\frac{1}{2}L^2\rho^2-B^2),
    \end{split}
\end{equation}
where $\tilde{\alpha}^t =\frac{1}{M}\sum_{i=1}^{M} |\alpha_i^t - \overline{\alpha}^t|$, a) uses $\langle a,b \rangle = -\frac{1}{2}\|a\|^2-\frac{1}{2}\|b\|^2 + \frac{1}{2}\|a - b\|^2$ and b) bases on assumption \ref{a1},\ref{a3}.\\
For the second term, we have
\begin{equation}
    \begin{split}
         & \Big \langle\nabla f(\mathbf{w}^{t}), \mathbb{E}\frac{1}{M}\sum_{i=1}^M \overline{\Delta}_i^t  \Big \rangle\\
         & \overset{a)}{\leq} \frac{- \overline{\alpha}^t\eta K}{2}\|\nabla f(\mathbf{w}^{t})\|^2 - \frac{\overline{\alpha}^t}{2K} \mathbb{E}\Big \|\frac{1}{\overline{\alpha}^t M}\sum_{i=1}^M \overline{\Delta}_i^t \Big \|^2
          + \frac{ \overline{\alpha}^t}{2}
         \underbrace{\mathbb{E}\Big \|\sqrt{K}\nabla f(\mathbf{w}^{t})- \frac{1}{ \overline{\alpha}^tM\sqrt{K}}\sum_{i=1}^M \overline{\Delta}_i^t \Big \|^2}_{\text{III}},
    \end{split}
\end{equation}
where a) uses $\langle a,b \rangle = -\frac{1}{2}\|a\|^2-\frac{1}{2}\|b\|^2 + \frac{1}{2}\|a - b\|^2$ and $0< \eta <1 $. Next, we bound III as follows:
\begin{equation}
    \begin{split}
         \text{III} &= K\mathbb{E}\Big \| \nabla f(\mathbf{w}^{t}) + \frac{1}{MK}\sum_{i=1}^M\sum_{k=0}^{K-1} \nabla \eta F_i(\mathbf{w}^{t,k}+\delta; \xi_i) \Big\|^2\\
        & \leq \frac{1}{M}\sum_{i=1}^M\sum_{k=0}^{K-1} \mathbb{E}\Big \| \eta (F_i(\mathbf{w}^{t,k}+\delta; \xi_i) - \nabla F_i(\mathbf{w}^{t,k}; \xi_i)) 
        + \eta (\nabla F_i(\mathbf{w}^{t,k}; \xi_i) - \nabla F_i(\mathbf{w}^{t})) + (1+\eta) \nabla F_i(\mathbf{w}^{t})\Big \| ^2\\
        & \overset{a)}{\leq} 3K\eta^2L^2 \Big( \rho^2 + \mathbb{E} \|\mathbf{w}^{t,k} - \mathbf{w}^{t}\|^2 + 2B^2 \Big)\\
        & \overset{b)}{\leq} 3K\eta^2L^2 \Big[ \rho^2 +  5K\eta^2 \Big(2L^2 \rho^2 \sigma_l^2+ 6K(3\sigma_g^2 + 6L^2 \rho^2 ) 
        + 6K\|\nabla f(\mathbf{w}^t)\|^2 \Big) + 24K^3 \eta^4 L^4 \rho^2 + B^2\Big],
    \end{split}
\end{equation}
where $0< \eta <1$, a) and b uses assumption \ref{a1}, \ref{a3} and lemma \ref{e_w}, respectively.\\
For II, we uses lemma \ref{Delta_average}. Then, combining Eq. 12-16, we have
\begin{equation}
    \begin{split}
         \mathbb{E} f(\mathbf{w}^{t+1}) 
        & \leq \mathbb{E} f(\mathbf{w}^t) + \eta \tilde{\alpha}_t K(\frac{1}{2}L^2\rho^2-B^2)
        - \frac{\overline{\alpha}^t\eta K}{2}\|\nabla f(\mathbf{w}^{t})\|^2 -
         \frac{\eta \overline{\alpha}^t}{2K} \mathbb{E}\Big \|\frac{1}{\eta\overline{\alpha}^t M}\sum_{i=1}^M \overline{\Delta}_i^t \Big \|^2 \\
        & + \frac{3\overline{\alpha}^t\eta^2L^2K}{2}\Big[ \rho^2 +  5K\eta^2 \Big(2L^2 \rho^2 \sigma_l^2+ 6K(3\sigma_g^2 + 6L^2 \rho^2 )  + 6K\|\nabla f(\mathbf{w}^t)\|^2 \Big)\\
        &  + 24K^3 \eta^4 L^4 \rho^2 + B^2\Big ]
         + \frac{3\eta^2KL(L^2\rho^2+B^2)}{2} + \frac{L\sigma^2C^2d}{2m^2}.
    \end{split}
\end{equation}
When $\eta \leq \frac{1}{3\sqrt{KL}}$, the inequality is 
\begin{equation}
    \begin{split}
         \mathbb{E} f(\mathbf{w}^{t+1}) 
         & \leq \mathbb{E} f(\mathbf{w}^t) - \frac{ \overline{\alpha}^t \eta K}{2}\mathbb{E} \|\nabla f(\mathbf{w}^t)\|^2 + \frac{\tilde{\alpha}^t \eta KL^2\rho^2}{2}  + \frac{3\overline{\alpha}^t\eta^2 KL^2\rho^2}{2}
        -\tilde{\alpha}^t \eta KB^2 \\
        & + \frac{15\overline{\alpha}^t K\eta^4L^2}{2} \Big(2L^2 \rho^2 \sigma_l^2 + 6K(3\sigma_g^2 + 6L^2 \rho^2 )  + 6K\|\nabla f(\mathbf{w}^t)\|^2 \Big) + 36\eta^6K^4L^6\rho^2 \\
        & + \frac{3\eta^2KL(L^2 \rho^2+B^2)}{2} + \frac{L\sigma^2C^2d}{2m^2}.
    \end{split}
\end{equation}
Sum over $t$ from $1$ to $T$, we have
\begin{equation}
    \begin{split}
         \frac{1}{T}\sum_{t=1}^T \mathbb{E} \Big[\overline{\alpha}^t\|f(\mathbf{w}^t)\|^2\Big] & \leq \frac{2L(f({\bf w}^{1})-f^{*})}{\sqrt{KT}} + \frac{1}{T}\sum_{t=1}^T \tilde{\alpha}^t  L^2\rho^2  - 2 \tilde{\alpha}^t  B^2 + 30\eta^2 L^2\frac{1}{T}\sum_{t=1}^T \overline{\alpha}^t \Big(2L^2 \rho^2 \sigma_l^2 + 6K(3\sigma_g^2 + 6L^2 \rho^2 ) \Big)\\
        & + 72\eta^4K^3L^6\rho^2+3\eta L(L^2 \rho^2+B^2)  + \frac{L\sigma^2C^2d}{\eta m^2K}
    \end{split}
\end{equation}
Assume the local adaptive learning rate satisfies $\eta=\mathcal{O}({1}/{L\sqrt{KT}})$, both $\frac{1}{T}\sum_{t=1}^T \tilde{\alpha}^t $ and $\frac{1}{T}\sum_{t=1}^T \overline{\alpha}^t $ are two important parameters for measuring the impact of clipping. Meanwhile, both $\frac{1}{T}\sum_{t=1}^T \tilde{\alpha}^t $ and $\frac{1}{T}\sum_{t=1}^T \overline{\alpha}^t $ are also bounded by $1$. Then, our result is
\begin{equation}
    \begin{split}
            & \frac{1}{T} \sum_{t=1}^T
    \mathbb{E}\left[\overline{\alpha}^{t}\left\|\nabla f\left(\mathbf{w}^{t}\right)\right\|^{2}\right]   \leq  
    \underbrace{\mathcal{O}\left(\frac{2L(f({\bf w}^{1})-f^{*})}{\sqrt{KT}} + \frac{\sigma_{l}^2 L^2\rho}{KT}\right)}_{\text{From FedSAM}}   +
    \underbrace{
     \underbrace{\mathcal{O}\left(\sum_{t=1}^T( \frac{\overline{\alpha}^t  \sigma_{g}^2 }{T^2} + \frac{\tilde{\alpha}^t  L^2\rho^2 }{T} ) \right)}_{\text{Clipping}}
    + \underbrace{ \mathcal{O}\left(\frac{L^2 \sqrt{T}\sigma^2C^2d}{m^2\sqrt{K}} \right)}_{\text{Adding noise}}
    }_{\text{From operations for DP}} . 
    \end{split}
\end{equation}
Assume the perturbation amplitude $\rho$
proportional to the learning rate, e.g., $\rho = \mathcal{O}(\frac{1}{\sqrt{T}})$, we have
\begin{equation}
    \begin{split}
            & \frac{1}{T} \sum_{t=1}^T
    \mathbb{E}\left[\overline{\alpha}^{t}\left\|\nabla f\left(\mathbf{w}^{t}\right)\right\|^{2}\right]   \leq  
    \underbrace{\mathcal{O}\left(\frac{2L(f({\bf w}^{1})-f^{*})}{\sqrt{KT}} + \frac{ L^2\sigma_{l}^2}{KT^2}\right)}_{\text{From FedSAM}}   +
    \underbrace{
     \underbrace{\mathcal{O}\left(\sum_{t=1}^T( \frac{\overline{\alpha}^t  \sigma_{g}^2 }{T^2} + \frac{\tilde{\alpha}^t  L^2 }{T^2} ) \right)}_{\text{Clipping}}
    + \underbrace{ \mathcal{O}\left(\frac{L^2 \sqrt{T}\sigma^2C^2d}{m^2\sqrt{K}} \right)}_{\text{Adding noise}}
    }_{\text{From operations for DP}} . 
    \end{split}
\end{equation}
\end{proof}

\end{document}